%% file: main.tex
\documentclass[11pt]{article}
\input{temp_marco}

\pdfoutput=1

\title{Differentially Private Random Feature Model}
\author{Chunyang Liao, Deanna Needell, Hayden Schaeffer, Alexander Xue \\\\ Department of Mathematics, University of California, Los Angeles \\\\ Emails: \{liaochunyang, deanna, hayden, alexxue\}@math.ucla.edu}
\date{\today}

\begin{document}
\pagenumbering{arabic}
\maketitle

\begin{abstract}
Designing privacy-preserving machine learning algorithms has received great attention in recent years, especially in the setting when the data contains sensitive information. 
Differential privacy (DP) is a widely used mechanism for data analysis with privacy guarantees.
In this paper, we produce a differentially private random feature model. 
Random features, which were proposed to approximate large-scale kernel machines, have been used to study privacy-preserving kernel machines as well. 
We consider the over-parametrized regime (more features than samples) where the non-private random feature model is learned via solving the min-norm interpolation problem, and then we apply output perturbation techniques to produce a private model. 
We show that our method preserves privacy and derive a generalization error bound for the method. To the best of our knowledge, we are the first to consider privacy-preserving random feature models in the over-parametrized regime and provide theoretical guarantees. 
We empirically compare our method with other privacy-preserving learning methods in the literature as well. Our results show that our approach is superior to the other methods in terms of generalization performance on synthetic data and benchmark data sets. 
Additionally, it was recently observed that DP mechanisms may exhibit and exacerbate disparate impact, which means that the outcomes of DP learning algorithms vary significantly among different groups. 
We adopt excessive risk gap and statistical parity as the fairness metrics. Our results suggest, both theoretically and numerically, that random features have the potential to achieve better fairness.

\end{abstract}

\section{Introduction}
Over the last decades, data privacy has become a central concern due to the great amount of personal data, such as medical records, stored and used in ubiquitous applications.
Machine learning algorithms have been shown to have the ability to explore population level patterns; however, the outputs of machine learning algorithms may reveal sensitive personal information. Thus, one of the challenges for machine learning algorithms is how to protect the privacy of sensitive personal information in the training data. 
Towards this end, ($\epsilon, \delta$)-differential privacy (DP) \cite{Dwork2014DP} has emerged as a gold standard. Essentially, a randomized mechanism is differentially private if its (distribution of) outcomes remain roughly the same by changing or removing any single individual in a dataset. 

Despite the promise and increasing use of differential privacy, it was recently observed that this mechanism may exacerbate bias and unfairness for different groups, especially for underrepresented groups \cite{Bagdasaryan2019disparate, xu2021disparate, kuppam2020fairdecisionmakingusing}. There is of course no ``gold standard" definition for fairness. In the literature on fair(er) classification algorithms, a wide recognized concept for fairness is statistical parity, also known as demographic parity \cite{feldman2015disparate}. It implies that predictions of a classification model should be independent of the sensitive attributes, e.g., gender or race. Statistical parity can be adapted to the regression setting as well \cite{xian2024differentiallyprivatepostprocessingfair}, which requires the output distribution of a regression model conditioned on each group to be similar. While privacy and fairness have been extensively studied separately in the literature, the intersection of them has only recently gained attention. In \cite{cummings2019privacyfairness} and \cite{agarwal2021privacyfairness}, the authors showed that fairness and privacy are incompatible in the sense that no $\epsilon$-DP algorithm can generally guarantee group fairness on the training set, unless the hypothesis class is restricted to constant predictors.

In this paper, we study the random features model under the differentially private setting. The random features model, which was proposed to approximate large-scale kernel machines \cite{Rahimi2007RFM,Liu2022RandomFF}, has wide applications in regression problems \cite{xie2022RandomFF, saha2023harfe}, classification problems \cite{Sun2018RF}, signal processing \cite{richardson2024srmd}, and scientific computing \cite{chen2022RFM, Liu2023RF, Nelson2024operatorRFM, Liao2025Cauchy}. 
In the context of differential privacy, the random feature model is used to study differentially private learning with kernel methods, see \cite{Jain2013KernelDP, Rubinstein2012DP_kernel, PARK2023DP_kernel, Chaudhuri2011DifferentiallyPE}. 
Linear models, such as support vector machines (SVM), can solve nonlinear cases by utilizing a kernel function. 
The idea of kernel method is to map features to a Hilbert space where the inner product can be computed by a kernel function. 
However, kernel-based models have a fundamental weakness in achieving differential privacy. According to the celebrated representer theorem \cite{Bernhard2001representer}, the optimal model is given by a linear combination of kernel functions centered at the data points. 
Thus, even if the model is computed in a privacy-preserving way, it also reveals the training data. The random features method, which approximates the kernel function using random projections, can avoid the aforementioned problem. 

In addition to approximating kernel machines, the random features model can be viewed as a two-layer network where the single hidden layer is randomized then fixed and only the weights in the output layer are trainable. The training of the random features model reduces to solving a convex optimization problem. 
Recently, a several papers \cite{Jacot2018NTK,Ghorbani2019NTK, Chen2020NTK,Ju2021NTK} have studied the training dynamics of over-parametrized neural networks (where the network size is larger than the sample size) through neural tangent kernel (NTK). The NTK theory showed that the trainable parameter of over-parametrized neural networks is close to its initialization, and hence the training dynamics of over-parametrized neural networks can be studied by over-parametrized random features. 
Therefore, exploring applications and examining the theoretical properties of the random feature models under differentially private setting provides valuable insights into differentially private learning with over-parametrized neural networks.

\subsection{Contributions}
The goal of this paper is to construct differentially private random features model along with theoretical guarantees. 
We put ourselves in the over-parametrized regime where the number of random features is greater than the number of samples. We consider the random features model obtained by solving the min-norm interpolation problem. To the best of our knowledge, it is the first work on the differentially private min-norm interpolation random features model. 
The construction of the DP random features model follows the output perturbation mechanism and it comes with privacy and generalization guarantees. Moreover, we study the fairness of our proposed DP random features model via the notation of excessive risk gap and of statistical parity. Our numerical experiments show that our proposed model achieves better generalization performance and does not seem to suffer unfairness.
More precisely, our contributions are as follows:
\begin{enumerate}
\item We propose a differentially private over-parametrized random features model (Algorithm \ref{alg:DPRF}) based on a min-norm interpolator and the standard Gaussian mechanism. Our privacy guarantee relies on the sensitivity method, where we provide a novel bound on the sensitivity. Moreover, we provide a generalization error bound using the concentration results of random matrix.
\item We study the fairness of our proposed algorithm via the notation of excessive risk gap and of statistical parity. Following the analysis of \cite{tran2021differentially} that the excessive risk gap for each group is approximated by the input norms of the elements in each group, we show that the random features method, which maps the data to a high-dimensional space where all samples have the same norms, achieves similar excessive risk gaps across all groups.
\item We validate our results using experiments on synthetic data and two real data sets. Our results show that our proposed method achieves better generalization performance when compared to the state-of-the-art DP random features methods in \cite{Chaudhuri2011DifferentiallyPE, Wang2024DPRFM}. 
We also demonstrate that the proposed DP random features model achieves fairer outputs in terms of excessive risk gap and statistical parity.
\end{enumerate}

\subsection{Related work} 

{\bf DP random features model.} There is a large body of literature on privacy preserving kernel-based machine learning algorithms \cite{Jain2013KernelDP, Rubinstein2012DP_kernel, PARK2023DP_kernel, Chaudhuri2011DifferentiallyPE, Wang2024DPRFM}. 
In \cite{PARK2023DP_kernel, Rubinstein2012DP_kernel}, the authors studied a private kernel Support Vector Machine (SVM) via random feature approximation of a shift-invariant kernel. 
In \cite{Jain2013KernelDP, Chaudhuri2011DifferentiallyPE}, the authors studied the problem of privacy preserving learning using empirical risk minimization and then applied their results to produce privacy-preserving kernel-based models. 
Previous work assumed the empirical risk is strongly convex, and therefore a strongly convex regularization term ($\ell_2$ regularization for example) was added to the empirical risk function. 
Their theoretical guarantees depend on the choice of regularization parameter, which cannot be generalized to the interpolation regime. 
 Our analysis, however, focuses on the interpolation regime and does not rely on the regularization term.
In \cite{Wang2024DPRFM}, the authors focused on the differentially private stochastic gradient descent (SGD) algorithm based on random features.
The theoretical guarantees rely on technical assumptions on the kernel and a certain spectral decay rate on the second moment operator, which may be difficult to verify in practice. The obtained convergence rate depends on the sample size, which may be hard to obtain since increasing the sample size is hard in real-world applications. 
Our analysis does not rely on technical assumptions on the kernel and the obtained convergence rate depends on the number of parameters only. \\

\noindent {\bf Fairness and Privacy.} Designing private fair learning algorithms has been studied in prior work \cite{xu2021disparate, Bagdasaryan2019disparate, xian2024differentiallyprivatepostprocessingfair, feldman2015disparate, jagielski2019fair, tran2021differentially, rosenblatt2024simple}. Focusing on the classification setting, it was shown that the performance impact of privacy may be more severe on underrepresented groups, resulting in accuracy disparity \cite{Bagdasaryan2019disparate}. In \cite{xu2021disparate}, the authors studied how group sample size and group clipping bias affect the impact of differential privacy in DP stochastic gradient descent and how adaptive clipping for each group helps to mitigate the disparate impact. The trade-offs between privacy and fairness were studied in \cite{xian2024differentiallyprivatepostprocessingfair, tran2021differentially} in the regression setting.  However, the trade-offs between privacy, fairness, and utility remain to be understood.

\subsection{Paper Organization}
The remainder of the paper is organized as follows. In Section \ref{Sec:preliminary}, we present preliminaries on differential privacy, fairness notation, and the random features model. Our main algorithm and its theoretical guarantees are introduced in Section \ref{Sec:main}. In Section \ref{Sec:numerical}, we empirically explore the trade-offs achieved by our post-processing algorithm on synthetic datasets and real datasets. We conclude the paper with several future directions in Section \ref{Sec:discussion}.

\section{Preliminaries}
\label{Sec:preliminary}

\subsection{Notation}
In this paper, we let $\R$ be the set of all real numbers, $\R_+$ the positive real numbers, and $\C$ be the set of all complex numbers where $i=\sqrt{-1}$ denotes the imaginary unit. 
We define the set $[N]$ to be all natural numbers smaller than or equal to $N$, i.e. $\{1,2,\dots,N\}$. 
We use bold letters to denote vectors or matrices, and denote the identity matrix of size $n\times n$ by $\Ib_n$. 
For any two vectors $\xb,\yb\in\C^d$, the inner product is denoted by $\langle \xb,\yb\rangle = \sum_{j=1}^d x_j\Bar{y}_j$, where $\xb = [x_1,\dots,x_d]^\top$ and $\yb = [y_1,\dots,y_d]^\top$. For a vector $\xb\in\C^d$, we denote by $\|\xb\|_p$ the $\ell_p$-norm of $\xb$.
For a matrix $\Ab\in\C^{m\times N}$, its conjugate transpose is denoted by $\Ab^*$. If $\Ab\in\C^{m \times N}$ is full rank, then its Moore-Penrose inverse is $\Ab^\dagger = \Ab(\Ab\Ab^*)^{-1}$ if $m\leq N$ or $\Ab^\dagger = (\Ab^*\Ab)^{-1}\Ab^*$ if $m\geq N$. The (induced) $p$-norm of matrix $\Ab$ is written as $\|\Ab\|_p$. 

\subsection{Differential Privacy}

{\color{black} Differential privacy (DP) guarantees that the result of a data analysis or a query remains unchanged even when one record in the dataset is modified or removed. Therefore, it prevents any deductions about the inclusion or exclusion of any specific individual.}
Two datasets $D$ and $D'$ are called {\it neighboring datasets}, denoted by $D\sim D'$, if they differ by a single datum, that is, $D$ and $D'$ have only one data point in difference. 
{\color{black} Dataset $D'$ is usually obtained by modifying or removing a data sample from $D$. For our purpose, we consider the situation where one data sample is swapped by a new sample from the same distribution or one data point is removed.} The definition for classical differential privacy, developed in \cite{Dwork2014DP}, is given in Definition \ref{DP}.

\begin{definition}[$(\epsilon_p, \delta_p)$-Differential Privacy]
\label{DP}
A randomized mechanism $\Mc$ provides $(\epsilon_p, \delta_p)$-differential privacy if, for all pairs of neighboring datasets $D$ and $D'$, and all subsets $R$ of possible outputs:
\begin{equation*}
\Pbb(\Mc(D)\in R) \leq e^{\epsilon_p} \Pbb(\Mc(D')\in R) + \delta_p.
\end{equation*}
\end{definition}

The parameter $\epsilon_p$ describes the privacy loss of the randomized mechanism, with values close to 0 denoting strong
privacy. The parameter $\delta_p\in[0,1)$ captures the probability of failure of the algorithm to satisfy $\epsilon_p$-DP.

\begin{definition}[$\ell_2$-sensitivity]
The $\ell_2$-sensitivity of a function $f:\N^{|\Xc|} \to\R^k$ is:
\begin{equation*}
\Delta_2(f) = \max_{D\sim D'} \|f(D) - f(D')\|_2.
\end{equation*}
\end{definition}
The $\ell_2$-sensitivity $\Delta_2(f)$ of a real-valued function $f$ characterizes the maximum amount by which $f$ changes in two adjacent inputs.

\subsection{Fairness Metrics}
For fairness, we consider {\it excessive risk gap} and {\it statistical parity} as the fairness metrics. \\

\noindent {\bf Excessive Risk Gap.} We first introduce {\it excessive risk}, which defines the difference between the private and non-private risk functions. 
Let $\ell:\R\times\R\to\R_+$ be a loss function. Suppose the randomized model $f_\theta$ is parameterized by $\theta$, we define the risk function $L(\theta, D)$ for any given dataset $D$ as
\begin{equation}
\label{risk_function}
L(\theta, D) = \frac{1}{|D|} \sum_{(\xb_i,y_i)\in D} \ell(f_\theta(\xb_i), y_i).
\end{equation}
Then we give a formal definition of excessive risk.
\begin{definition}[Excessive Risk]
Let $\theta^\sharp$ be the minimizer of risk function, i.e. $\theta^\sharp = \argmin_{\theta} L(\theta, D)$. Denote $\hat{\theta}$ the private model parameter. Then the excessive risk is defined as
\begin{equation*}
R(\theta, D) = \E_{\hat{\theta}} L(\hat{\theta}, D) - L(\theta^\sharp, D),
\end{equation*}
where the expectation is defined over the randomness of the private mechanism.
\end{definition}
Excessive risk is widely adopted to measure utility in private learning \cite{wang2017DP, Zhang2017private}. Next we define excessive risk gap, which has been used as a metric in fairness analysis, see example \cite{tran2021differentially}.
\begin{definition}[Excessive Risk Gap]
\label{def_erg}
Denote $R(\theta, D)$ and $R(\theta, D_a)$the population-level excessive risk and the group-level excessive risk fro sensitive group $a$, respectively. Then excessive risk gap for sensitive group $a$ is defined as
\begin{equation*}
\xi_a = |R(\theta) - R_a(\theta)|.
\end{equation*}
\end{definition}
Fairness is achieved when $\xi_a=0$ for all groups $a$. Unfairness may occur when different groups have different excessive risk gaps.
Moreover, a larger excessive risk gap for a certain group implies that private estimator leads to a larger utility loss for this group. \\

\noindent {\bf Statistical Parity.} Statistical parity requires the output distribution of the randomized regressor $f$ conditioned on each sensitive group to be similar \cite{calders2009building}. The similarity between distributions is measured by the Kolmogorov-Smirnov distance, defined for the probability measures $p,q$ supported on $\R$ by
\begin{equation*}
D_{KS}(p,q) = \sup_{t\in\R} \left| \int_{-\infty}^t p(x) - q(x) dx \right|.
\end{equation*}
\begin{definition}[Statistical Parity]
\label{def:SP}
A regressor $f:X\times A \to\R$ satisfies $\alpha$-approximate statistical parity (SP) if 
\begin{equation*}
\Delta_{SP}(f) := \max_{a,a'\in A} D_{KS}(r_a, r_{a'}) \leq \alpha,
\end{equation*}
where $r_a$ is the distribution of regressor output $f(X,A)$ conditioned on $A=a$.
\end{definition}

\subsection{Random Feature Model}

{\color{black} Suppose that we are in the regression setting where $f:\R^d\to\C$ be the target function and $\{(\xb_j, y_j)\}_{j\in[m]}$ be the training samples such that $\xb_j$ are realizations of a random variable and $y_j = f(\xb_j)$. In other words, we assume that there are no observational errors.}
The main idea of random feature model is to draw $N$ i.i.d random features $\{\omegab_k\}_{k\in[N]} \subset \R^d$ from a probability distribution $\rho(\omegab)$, and then construct an approximation of target function $f$ taking the form 
\begin{equation}
\label{rf_train}
    f^\sharp(\xb) = \sum_{k=1}^N c_k^\sharp \exp(i\langle \omegab_k,\xb\rangle).
\end{equation}
The random feature model was proposed to approximate large-scale kernel machines \cite{Rahimi2007RFM}. Following the classical theorem from harmonic analysis, it provides the insight of shift-invariant kernel approximation:
\begin{theorem}[Bochner \cite{Rudin1994Fourier}]
A continuous shift-invariant kernel $k(\xb,\xb') = k(\xb-\xb')$ on $\R^d$ is positive definite if and only if $k(\delta)$ is the Fourier transform of a non-negative measure.
\end{theorem}
Bochner’s theorem guarantees that the Fourier transform of kernel $k(\delta)$ is a proper probability distribution if the kernel is scaled properly. Denote the probability distribution by $\rho(\omegab)$, we have 
\begin{equation}
\label{kernel}
k(\xb,\xb') = \int_{\R^d} \exp(i\langle \omegab_k,\xb-\xb'\rangle)d\rho(\omegab) = \int_{\R^d} \exp(i\langle \omegab_k,\xb\rangle) \overline{\exp(i\langle \omegab_k,\xb'\rangle)}d\rho(\omegab).    
\end{equation}
Using Monte Carlo sampling techniques, we randomly generate $N$ i.i.d samples $\{\omegab_k\}_{k\in[N]}$ from $\rho(\omegab)$ and define a {\it random Fourier feature map} $\phi:\R^d\to\C^N$ as
$$\phi(\xb) = \frac{1}{\sqrt{N}}\left[\exp(i\langle \omegab_1, \xb\rangle), \dots, \exp(i\langle\omegab_N,\xb\rangle) \right]^T \in\C^N.$$
The kernel function can be approximated by $k(\xb,\xb') \approx \langle \phi(\xb), \phi(\xb')\rangle$. 
To produce feature map $\phi$ mapping data to finite-dimensional space $\R^N$, we can use {\it random cosine features} \cite{Rahimi2008RFM}. Setting random feature $\omegab\in\R^d$ generated from $\rho(\omegab)$ and $b\in\R$ sampled from uniform distribution on $[-\pi,\pi]$, the random cosine feature is defined as 
$\phi(\xb,\omegab,b) = \cos(\langle \omegab,\xb\rangle + b)$. 
For the sake of simplicity, we consider the random Fourier features in this paper. The results can be generalized to random cosine features, and therefore approximate real-valued function $f:\R^d\to\R$.

Training the random feature model \eqref{rf_train} is equivalent to finding the coefficient vector $\cb^\sharp\in\C^N$. Let $\Ab\in\C^{m\times N}$ be the random feature matrix defined component-wise by $\Ab_{j,k}=\exp(i\langle \omegab_k, \xb_j\rangle)$ for $j\in[m]$ and $k\in[N]$. In this paper, we consider the over-parametrized regime ($N\geq m$), where the coefficient vector $\cb^\sharp\in\C^N$ is trained by solving the following min-norm interpolation problem:
\begin{equation}
\label{ridgeless}
    \cb^\sharp \in \argmin_{\Ab\cb = \yb} \|\cb\|_2.
\end{equation}
{\color{black} The solution is given by $\cb^\sharp = \Ab^*(\Ab\Ab^*)^{-1}\yb$.
In Section \ref{Sec:main}, we will show that the Gram matrix $\Ab\Ab^*$ is invertible with high probability, and hence the Moore-Penrose inverse $\Ab^\dagger$ is well-defined.}
Since $\cb^\sharp=\Ab^\dagger\yb$ can be viewed as the limit of ridge coefficients $\cb^\sharp_\lambda$ that are obtained by solving the ridge regression problem
\begin{equation}
    \cb^\sharp_\lambda = \argmin_{\cb\in\C^N} \frac{1}{m}\| \Ab\cb - \yb\|_2^2 + \lambda\|\cb\|_2^2,
\end{equation}
then problem \eqref{ridgeless} is also called the ridgeless regression problem.

\section{Main Results}
\label{Sec:main}
In this section, we present the main algorithm which produces differentially private random feature models, and provide theoretical guarantees for privacy and generalization.

\subsection{Algorithm}
Consider a dataset $D=\{ (\xb_j, y_j)\in \R^d\times \R: j\in[m]\}$ consisting of $m$ samples where labels satisfy $y_j = f(\xb_j)$ for some unknown function $f:\R^d\to\R$. Our goal is to approximate $f$ by using a random feature model \eqref{rf_train} whose coefficients are trained by solving the min-norm interpolation problem \eqref{ridgeless}. Denote the solution to \eqref{ridgeless} by $\cb^\sharp$. We are not able to release it as it encodes information about the dataset $D$. Instead, we provide coefficients $\hat{\cb}$ whose performance on approximation is good, while at the same time ensuring that it satisfies differential privacy with parameters $\epsilon_p$ and $\delta_p$. 

The main idea of our algorithm follows  standard output perturbation techniques: to obtain the DP estimator, we first compute a non-private estimator, then add well chosen noise $\zb$ for privacy. This takes inspiration from \cite{Wang2024DPRFM, Chaudhuri2011DifferentiallyPE}. However, we choose random noise from the normal distribution with mean 0 and variance $\sigma^2$ (computed explicitly in our algorithm) instead of using noise with density $\nu(b) \propto \exp(\eta\|b\|_2)$ as \cite{Chaudhuri2011DifferentiallyPE}. We  show that our method has better accuracy guarantees, and experimental results support our observation as well.

\begin{algorithm}
\caption{Differentially private random feature model based on output perturbation}\label{alg:DPRF}
\begin{algorithmic}
\State\textbf{Inputs:} Non-private coefficient vector $\cb^\sharp$, privacy budget $\epsilon_p$, privacy parameter $\delta_p$, number of random features $N$, and parameter $\eta$
\State\textbf{Outputs:} DP coefficient vector $\hat{\cb}$

\State 1. Compute $\Delta = \frac{2}{\sqrt{N(1-2\eta)}}$ and  $\sigma^2 = \frac{2\ln(1.25/\delta_p)\Delta^2}{\epsilon_p^2}$

\State 2. Sample $\zb\in\R^N$ from normal distribution $\Nc(0,\sigma^2\Ib_N)$

\State 3. Return private coefficient vector $\Hat{\cb} = \cb^\sharp + \zb$
\end{algorithmic}
\end{algorithm}

\subsection{Privacy Guarantee}
In this section, we present the privacy guarantee on our proposed Algorithm \ref{alg:DPRF}. The privacy guarantee follows {\color{black} from} 
Theorem 3.22 in \cite{Dwork2014DP} directly.

\begin{theorem}[Gaussian Mechanism]
Let $\epsilon_p\in(0,1)$ be arbitrary. For $c^2\geq 2\ln(1.25/\delta_p)$, the Gaussian Mechanism 
\begin{equation*}
\Mc_{Gaussian}(f, X, \epsilon_p, \delta_p) = f(X) + \Nc(0, \sigma^2)
\end{equation*}
is $(\epsilon_p, \delta_p)$-differentially private with parameter $\sigma\geq c\Delta_2(f)/\epsilon_p$.
\end{theorem}

To show the privacy guarantee, it remains to quantify the $\ell_2$-sensitivity for non-private coefficient $\cb^\sharp$, which relies on concentration property of random feature matrix. 
We state the concentration property in Theorem \ref{Concentration_RF}, whose proof can be found in \cite{chen2022concentration}. 
Before we present the result, we first state assumptions on the samples and random features. 
Specifically, the samples $\{\xb_j\}_{j\in[m]}$ are assumed to be independent sub-gaussian random variables and the random features $\{\omegab_k\}_{k\in[N]}$ are i.i.d samples generated from a multivariate normal distribution. 

\begin{assumption}
\label{assumption}
We assume that the random features $\{\omegab_k\}_{k\in[N]} \subset \R^d$ are sampled from $\Nc(0,\sigma^2\Ib_d)$, and $\{\xb_j\}_{j\in[m]} \subset \R^d$ are data points such that the components of $\xb_j$ are independent mean-zero sub-gaussian random variables with the same variance $\gamma^2/d$ and the same sub-gaussian parameters $\beta, \kappa$. The random feature matrix $\Ab$ is defined component-wise by $\Ab_{j,k} = \exp(i\langle\omegab_k,\xb_j\rangle)$
\end{assumption}

\begin{theorem}[{\bf Concentration result of Random Features Matrix}]
\label{Concentration_RF}
Let data samples $\{\xb_j\}_{j\in[m]} \subset \R^d$, random features $\{\omegab_k\}_{k\in[N]} \subset \R^d$ and random feature matrix $\Ab$ satisfy Assumption \ref{assumption}.
Then there exist a constant $C_1 > 0$ (depending only on sub-gaussian parameters) and a universal constant $C_2 > 0$ such that if the following conditions hold
\begin{equation*}
\begin{aligned}
& d \geq C_1\log\left(\frac{m}{\delta}\right), \\
& \gamma^2\sigma^2 \geq 4\log\left(\frac{2m}{\eta}\right), \\
& N\geq C_2\eta^{-2}m\log\left(\frac{2m}{\delta}\right),
\end{aligned}
\end{equation*}
for some $\delta,\eta\in(0,1)$, then we have
\begin{equation*}
\left\| \frac{1}{N}\Ab\Ab^* - \Ib_m \right\|_2 \leq 2\eta
\end{equation*}
with probability at least $1-3\delta$.
\end{theorem}
As a direct consequence of Theorem \ref{Concentration_RF}, all the eigenvalues of matrix $\frac{1}{N}\Ab\Ab^*$ are close to 1, i.e.
\begin{equation*}
\left| \lambda_k \left(\frac{1}{N}\Ab\Ab^*\right) - 1 \right| \leq 2\eta
\end{equation*}
if the conditions in Theorem \ref{Concentration_RF} are satisfied.
It also implies that the matrix $\Ab\Ab^*$ is invertible with high probability, and hence Moore-Penrose inverse of matrix $\Ab$ is well-defined.
Applying this result gives a bound on the $\ell_2$-sensitivity, and then provides the privacy guarantee of our proposed algorithm.
\begin{lemma}
Suppose that dataset $D = \{ (\xb_j, y_j)\in \R^d\times \R: j\in[m]\}$ and random features $\{\omegab_k\}_{k\in[N]} \subset \R^d$ satisfy Assumption \ref{assumption} and the conditions in Theorem \ref{Concentration_RF} hold. Suppose that labels $\yb$ satisfy $\|\yb\|_2\leq1$, then the $\ell_2$-sensitivity satisfies
\begin{equation*}
\Delta_2(\cb^\sharp) \leq \frac{2}{\sqrt{N(1-2\eta)}}
\end{equation*}
with probability at least $1-3\delta$ for some $\delta\in(0,1)$.
\end{lemma}
\begin{proof}
Recall that the non-private coefficient vector $\cb^\sharp = \Ab^\dagger\yb$. For any two neighboring datasets $D$ and $D'$, we have
{\color{black} \begin{equation*}
\begin{aligned}
\|\cb^\sharp(D) - \cb^\sharp(D')\|_2 \leq& \|\cb^\sharp(D)\|_2 + \|\cb^\sharp(D')\|_2 \leq \|\Ab_D^\dagger\|_2\|\yb\|_2 + \|\Ab_{D'}^\dagger\|_2\|\yb\|_2 \\
\leq& \frac{2\|\yb\|_2}{\sqrt{N\lambda_{\min}(\frac{1}{N}\Ab\Ab^*)}} \leq \frac{2}{\sqrt{N(1-2\eta)}}
\end{aligned}
\end{equation*}}
with probability at least $1-3\delta$. The last inequality holds by applying the consequence of Theorem \ref{Concentration_RF} and the assumption that $\|\yb\|_2\leq1$. Taking the supremum over all neighboring datasets $D$ and $D'$ leads to the desired result.
\end{proof}

Unlike the traditional analysis on the $\ell_2$ sensitivity that relies on the sample size $m$ and regularization parameter, our bound depends on the number of features $N$ instead. 
Increasing $N$ results in smaller $\ell_2$-sensitivity, which leads to smaller variance of additive random noise. 
Hyper-parameter $\eta$ characterizes the condition number of matrix $\Ab\Ab^*$, which is chosen to be a fixed constant in our numerical experiments. 
Our theory suggests that smaller $\eta$ is favorable since it results in smaller $\ell_2$-sensitivity.
However, this will lead to larger number of features $N$ and larger variance of random features $\sigma^2$, see the conditions in Theorem \ref{Concentration_RF}.


\subsection{Generalization Error}
In this section, we will derive an upper bound for the generalization error of our proposed differentially private random features model. 
We first define the target function class that can be approximated by random features model.
Let $f:\R^d\to\C$ be the target function belonging to the function class
\begin{equation}
\label{function_space}
\Fc(\rho) := \left\{ f(\xb) = \E_\rho \left[ \alpha(\omegab)\exp(i\langle \omegab,\xb\rangle) \right] : \|f\|^2_\rho = \E_\rho \left[\alpha(\omegab)^2 \right] <\infty \right\}
\end{equation}
where $\rho(\omegab)$ is a probability distribution defined on $\R^d$.
Function space $\Fc(\rho)$ is indeed a Reproducing kernel Hilbert space with associated kernel function defined in \eqref{kernel} via random Fourier features.

Denote the private random feature model by $\hat{f}:\R^d\to\C$ with private coefficient vector $\hat{\cb}\in\R^N$. Our goal is to bound the generalization error
\begin{equation*}
\|f-\hat{f}\|_{L^2(\mu)} = \left( \int_{\R^d} \left|f(\xb) - \hat{f}(\xb)\right|^2 d\mu \right)^{1/2},
\end{equation*}
where $\mu$ is defined by the sampling measure for the data $\xb$.
Our main result is stated in the following theorem, which characterizes the generalization error bound for differentially private random features model.
\begin{theorem}[{\bf Generalization Error Bound for DP Random Features Model}]
\label{main}
For any function $f\in\Fc(\rho)$, data samples $\{\xb_j\}_{j\in[m]}$, random features $\{\omegab_k\}_{k\in[N]}$ and random feature matrix $\Ab$ satisfy Assumption \ref{assumption}, if the conditions in Theorem \ref{Concentration_RF} hold, then the generalization error bound for the private random feature model $\hat{f}$ 
\begin{equation*}
\begin{aligned}
& \|f-\hat{f}\|_{L^2(\mu)} \\
\leq & \left( \frac{14\log(2/\delta)}{\sqrt{N}} + \frac{28\left(2m\log(1/\delta)\right)^{1/4}\log(2/\delta)}{\sqrt{N(1-2\eta)}} + \left(\frac{32\log(1/\delta)}{m}\right)^{1/4}\sqrt{\log(2/\delta)} \right) \|f\|_\rho + \\
& \sqrt{N}\left( \sqrt{\frac{(1+2\eta)}{m}} + \left(\frac{2\log(1/\delta)}{m}\right)^{1/4} \right) \left(\frac{2\log(1.25/\delta_p)}{(1-\eta)\epsilon_p^2} + \frac{8\sqrt{2}}{\sqrt{N}(1-\eta)\epsilon_p}\sqrt{\ln\left(\frac{1}{\delta}\right)}\right)^{1/2}
\end{aligned}
\end{equation*}
holds with probability at least $1-9\delta$.
\end{theorem}
\begin{proof}[Sketch of proof]
Denote the non-private random feature model by $f^\sharp$ whose coefficients are trained by solving the min-norm interpolation problem \eqref{ridgeless}. We decompose the generalization error into
\begin{equation}
\label{deco_gen_err}
\|f-\hat{f}\|_{L^2(\mu)} \leq \|f-f^\sharp\|_{L^2(\mu)} + \|f^\sharp - \hat{f}\|_{L^2(\mu)}
\end{equation}
and bound each term separately. The first term is the standard generalization error of random features model, which appears in recent work \cite{HASHEMI2023generalization, chen2024conditioning}. 
The second term characterizes the difference between non-private and private random feature models, which can be bounded by applying the sub-exponential tail bound. 
\end{proof}

n the rest of this section, we provide an error bound for each term in the right-hand side of \eqref{deco_gen_err}. We first consider the term $\|f-f^\sharp\|_{L^2(\mu)}$, which characterizes the generalization error of the non-private random features model.
Our analysis follows the structure in \cite{HASHEMI2023generalization, chen2024conditioning}, where the authors provided an upper bound for generalization error using the concentration property of random feature matrix. 
However, their results only hold for functions $f$ belonging to a subset of $\Fc(\rho)$.
Our technical contribution is taking a different ``best" random feature approximation $f^*$ in our analysis, which allows us to approximate all functions in $\Fc(\rho)$.
Specifically, we define the best random feature approximation $f^*(\xb)$ as
\begin{equation}
\label{rf_best}
f^*(\xb) = 
\frac{1}{N}\sum_{k=1}^N \alpha_{\leq T}(\omegab_k) \exp(i\langle \omegab_k,\xb\rangle)
\end{equation}
where $\omegab_k$ are i.i.d samples from distribution $\rho(\omegab)$ and $\alpha_{\leq T}(\omegab) = \alpha(\omegab)\indicator_{\left|\alpha(\omegab)\right|\leq T}$ for some parameter $T$ to be determined. We  introduce $\alpha_{>T} = \alpha(\omegab) - \alpha_{\leq T}(\omegab) $, i.e.
$$
\alpha_{\leq T}(\omegab) = \begin{cases}
    \alpha(\omegab) \quad & \mbox{ if } \left|\alpha(\omegab) \right| \leq T \\
    0 \quad & \mbox{ otherwise }
\end{cases} \quad \mbox{ and } \quad
\alpha_{> T}(\omegab) = \begin{cases}
    \alpha(\omegab) \quad & \mbox{ if } \left|\alpha(\omegab) \right| > T \\
    0 \quad & \mbox{ otherwise }.
    \end{cases}
$$
It is easy to check that 
$$\E_{\omegab} f^*(\xb) = \E_{\omegab} \left[ \alpha_{\leq T}(\omegab)\exp(i\langle \omegab,\xb\rangle) \right].$$
With the best random feature approximation $f^*$, we provide generalization error bounds for the non-private random feature model.

\begin{theorem}[{\bf Generalization Error Bound for Non-private Random Features Model.}]
\label{error_non_private}
For any function $f\in\Fc(\rho)$, data samples $\{\xb_j\}_{j\in[m]}$, features $\{\omegab_k\}_{k\in[N]}$ and random feature matrix $\Ab$ that satisfy Assumption \ref{assumption}, if the conditions in Theorem \ref{Concentration_RF} hold, then the generalization error for the non-private random feature model $f^\sharp$ is bounded by
\begin{equation*}
\begin{aligned}
& \|f-f^\sharp\|_{L^2(\mu)} \\
\leq &
\left( \frac{14\log(2/\delta)}{\sqrt{N}} + \frac{28\left(2m\log(1/\delta)\right)^{1/4}\log(2/\delta)}{\sqrt{N(1-2\eta)}} + \left(\frac{32\log(1/\delta)}{m}\right)^{1/4}\sqrt{\log(2/\delta)} \right) \|f\|_\rho
\end{aligned}
\end{equation*}
with probability at least $1-8\delta$.
\end{theorem}
\begin{proof}
The full proof is given in Appendix. The proof follows the similar techniques of \cite{HASHEMI2023generalization, chen2024conditioning}. Specifically, we decompose the generalization error for the non-private random feature model $f^\sharp$ into 
\begin{equation*}
\|f-f^\sharp\|_{L^2(\mu)} \leq \| f - f^* \|_{L^2(\mu)} + \|f^* - f^\sharp\|_{L^2(\mu)}.
\end{equation*}
The first term is the {\bf approximation error} which measures the difference between the target function and the "best" random features model. The second term is the {\bf estimation error}. To bound the second term, we utilize the concentration property of random feature matrix (Theorem \ref{Concentration_RF}) as well.
\end{proof}

Now, we consider the second term $\|f^\sharp - \hat{f}\|_{L^2(\mu)}$ in \eqref{deco_gen_err} that characterizes the difference between the private random features model and the non-private random features model.
Notice that the difference between the non-private coefficient vector $\cb^\sharp$ and the private coefficient vector $\hat{\cb}$ is indeed $\zb\in\R^N$, which is generated from multivariate normal distribution $\Nc(0,\sigma^2\Ib_N)$ with the variance $\sigma^2 = \frac{8\ln(1.25/\delta_p)}{N\epsilon_p^2(1-2\eta)}$ for some $\eta\in(0,1)$. 
In the following theorem, we present an error bound on the second term $\|f^\sharp - \hat{f}\|_{L^2(\mu)}$.

\begin{theorem}
Let $f^\sharp$ and $\hat{f}$ be the non-private and private random feature models, respectively. Let data samples $\{\xb_j\}_{j\in[m]} \subset \R^d$, features $\{\omegab_k\}_{k\in[N]} \subset \R^d$ and random feature matrix $\Ab$ satisfy Assumption \ref{assumption}. If the conditions in Theorem \ref{Concentration_RF} hold, then we have 
\begin{equation*}
\begin{aligned}
& \|f^\sharp - \hat{f}\|_{L^2(\mu)} \\
\leq& \sqrt{N}\left( \sqrt{\frac{(1+2\eta)}{m}} + \left(\frac{2\log(1/\delta)}{m}\right)^{1/4} \right) \left(\frac{2\log(1.25/\delta_p)}{(1-2\eta)\epsilon_p^2} + \frac{8\sqrt{2}}{\sqrt{N}(1-2\eta)\epsilon_p}\sqrt{\ln\left(\frac{1}{\delta}\right)}\right)^{1/2}
\end{aligned}
\end{equation*}
with probability at least $1-5\delta$. 
\end{theorem}
\begin{proof}
We will use McDiarmid's inequality. Draw i.i.d samples $\Uc = \{\ub_j\}_{j\in[m]}$ from the density $\mu$, and define random variable 
\begin{equation*}
\begin{aligned}
v(\ub_1,\dots,\ub_m) &:= \|f^\sharp - \hat{f}\|^2_{L^2(\mu)} - \frac{1}{m}\sum_{j=1}^m \left|f^\sharp(\ub_j)-\hat{f}(\ub_j)\right|^2 \\
&= \frac{1}{m} \E_{\Uc} \left[ \sum_{j=1}^m \left|f^\sharp(\ub_j)-\hat{f}(\ub_j)\right|^2 \right] - \frac{1}{m}\sum_{j=1}^m \left|f^\sharp(\ub_j)-\hat{f}(\ub_j)\right|^2.
\end{aligned}
\end{equation*}
Thus, $v$ has mean zero, i.e. $\E v =0$. The points $\ub_j\in\Uc$ are i.i.d samples, independent of the samples used to train the non-private coefficient vector $\cb^\sharp$, and independent of the random noise $\zb$ added to make private coefficient vector $\hat{\cb}$. Therefore, the points $\ub_j\in\Uc$ are independent of $\cb^\sharp$ and $\Hat{\cb}$. To apply McDiarmid's inequality, we first show that $v$ is stable under a perturbation of any one of its coordinates. Perturbing the $j$-th random variable $\ub_j$ to $\Tilde{\ub}_j$ leads to
\begin{equation*}
\begin{aligned}
& |v(\ub_1,\dots,\ub_j,\dots,\ub_m) - v(\ub_1,\dots,\Tilde{\ub}_j,\dots,\ub_m) | \\
=& \frac{1}{m} \left| |f^\sharp(\ub_j)-\hat{f}(\ub_j)|^2 - |f^\sharp(\Tilde{\ub}_j)-\hat{f}(\Tilde{\ub}_j)|^2 \right| \\
\leq& \frac{2N}{m} \|\cb^\sharp - \hat{\cb}\|_2^2 \leq \frac{2N}{m} \|\zb\|_2^2 := \Delta_v,
\end{aligned}
\end{equation*}
where the third line holds by the Cauchy-Schwarz inequality and 
\begin{equation*}
|f^\sharp(\ub_j)-\hat{f}(\ub_j)|^2 = \left|\sum_{k=1}^N (c_k^\sharp - \hat{c}_k)\exp(i\langle \omegab_k, \ub_j\rangle)\right|^2 \leq N\|\cb^\sharp - \hat{\cb}\|_2^2.
\end{equation*}
Next, we apply McDiarmid's inequality $\Pbb(v \geq t) \leq \exp\left(-\frac{2t^2}{m\Delta_v^2}\right)$, by selecting
\begin{equation*}
t = \frac{\sqrt{2\log\left(\frac{1}{\delta}\right)}N}{\sqrt{m}}\|\zb\|_2^2,
\end{equation*}
which leads to
\begin{equation*}
\|f^\sharp - \hat{f}\|^2_{L^2(\mu)} \leq \frac{1}{m}\sum_{j=1}^m \left|f^\sharp(\ub_j)-\hat{f}(\ub_j)\right|^2 +\frac{\sqrt{2\log\left(\frac{1}{\delta}\right)}N}{\sqrt{m}}\|\zb\|_2^2
\end{equation*}
with probability at least $1-\delta$ (with respect to the draw of $\Uc$). 
Define the random feature matrix $\Tilde{\Ab}$ element-wise by $\Tilde{\Ab}_{j,k} = \exp(i\langle \omegab_k, \ub_j\rangle)$, we have
\begin{equation*}
\begin{aligned}
& \frac{1}{m}\sum_{j=1}^m \left|f^\sharp(\ub_j)-\hat{f}(\ub_j)\right|^2 = \frac{1}{m} \left\|\Tilde{\Ab}(\cb^\sharp - \hat{\cb})\right\|_2^2 \\
\leq& \frac{1}{m}\|\Tilde{\Ab}\|_2^2 \|\zb\|_2^2 \leq \frac{N}{m}\lambda_{\max}\left(\frac{1}{N}\Tilde{\Ab}\Tilde{\Ab}^*\right)\|\zb\|_2^2 \leq \frac{N(1+2\eta)}{m}\|\zb\|_2^2
\end{aligned}
\end{equation*}
where the last inequality holds with probability at least $1-3\delta$ using Theorem \ref{Concentration_RF}. 
we apply Lemma \ref{tail_z} to estimate $\|\zb\|_2^2$. Taking the square root for both sides leads to the desired bound.
\end{proof}

\begin{lemma}
\label{tail_z} 
Let $\zb\in\R^N$ be a random vector $\zb\sim \Nc(0, \sigma^2\Ib_N)$ where $\sigma^2 = \frac{8\ln(1.25/\delta_p)}{N\epsilon_p^2(1-2\eta)}$, then we have 
\begin{equation*}
\|\zb\|_2^2 \leq \frac{2\log(1.25/\delta_p)}{(1-2\eta)\epsilon_p^2} + \frac{8\sqrt{2}}{\sqrt{N}(1-2\eta)\epsilon_p}\sqrt{\ln\left(\frac{1}{\delta}\right)}
\end{equation*}
with probability at least $1-\delta$.
\end{lemma}
\begin{proof}
Notice that $\zb_k^2$ is a sub-exponential random variable with parameters $(\nu_k, \alpha_k) = (2\sigma^2, 4\sigma^2)$, then $\|\zb\|_2^2$ is also a sub-exponential random variable with parameters $\nu = \sqrt{\sum_{k=1}^N \nu_k^2}$ and $\alpha=\alpha_k$. Then, we apply standard sub-exponential tail bound, see \cite[Proposition 2.9]{Wainwright_2019}, to obtain the desired estimation of $\|\zb\|_2^2$. 
\end{proof}

The output perturbation mechanisms were also employed in \cite{Chaudhuri2011DifferentiallyPE, acharya2024personalized}, where the added random noise $\zb$ follows a probability distribution with density function $\propto \exp(-\xi\|b\|_2)$.  
To sample random variable $\zb$ with density function $\propto \exp(-\xi\|b\|_2)$, it suffices to i) sample the radius $R$ from Gamma($N$,$\eta$), then ii) sample $Y$ uniformly at random from the $\ell_2$-sphere of radius 1, and set $\zb=RY$. Moreover, we can show that $\|\zb\|_2$ follows a Gamma distribution Gamma($N$,$\xi$) where $\xi$ is selected as $\xi = \epsilon_p / \Delta(\cb^\sharp)$ to guarantee $(\epsilon_p,0)$-differential privacy. By Chebyshev's inequality, we have that 
$$\|\zb\|_2 \leq \sqrt{\frac{2}{(1-2\eta)\epsilon_p^2}}\left( \sqrt{N}+\sqrt{\frac{2}{\delta}} \right)$$
with probability at least $1-\delta/2$. It is obvious to see that this upper bound is worse than the bound in Lemma \ref{tail_z} in terms of privacy budget $\epsilon_p$ and number of features $N$. Our experimental results also support this observation, see comparison in Section \ref{Sec:numerical}.

\subsection{Fairness Guarantee}
In this section, we aim to show that random feature model has the potential to remove disparate impact via the notation of excessive risk gap. 
It has been observed that the unfairness, such as disparate impact, exists in the DP mechanism \cite{Bagdasaryan2019disparate, xu2021disparate, rosenblatt2024simple}. 
The disparate impact refers to the phenomenon that the performance of a differentially private model for underrepresented classes and subgroups tends to be worse than that of the original, non-private model. 
Consider the regression problem (it can be generalized to any supervised learning problem) with mean-squared loss, we can explicit write the risk function \eqref{risk_function} as 
\begin{equation*}
L(\theta, D) = \frac{1}{|D|} \sum_{(\xb_i,y_i)\in D} (f_\theta(\xb_i) - y_i)^2.
\end{equation*}
Following the analysis in \cite{tran2021differentially}, for a fixed privacy budget $\epsilon$, the excessive risk gap for linear models mainly depends on the trace of Hessian matrix $\Hb^a = \nabla_{\theta^\sharp}^2 L(\theta^\sharp, D_a)$ of risk function, at the optimal parameter vector $\theta^\sharp$. 
Suppose that the Gaussian mechanism is applied to produce $(\epsilon_p, \delta_p)$-differentially private estimator $\hat{\theta}$, i.e adding noise drawn from normal distribution $\Nc(0,\sigma^2)$ to each entry of the optimal model parameter $\theta^\sharp$. Theorem 1 in \cite{tran2021differentially} indicates that the excessive risk gap $\xi_a$ (see Definition \ref{def_erg}) is approximated by
\begin{equation}
\label{excessive_risk_gap}
\xi_a \approx \frac{1}{2}\sigma^2 |\Tr(\Hb^a) - \Tr(\Hb)|.
\end{equation}
Moreover, the trace of the Hessian matrix $\Hb^a$ depends solely on the input norms of the elements in $D_a$ for each group $a$ since $\Tr(\Hb^a) = \E_{\xb\in D_a}\|\xb\|^2$, which highlights the relation between fairness and the average input norms of different group elements. 
When these norms are substantially different from one another they will impact their respective excessive risks differently. 
Based on this observation, Corollary 2 in \cite{tran2021differentially} suggests that better fairness may be achieved by normalizing the input values for each group independently in the regression setting.
The random feature method can be viewed as a normalization approach using random projections, and therefore it helps to achieve better fairness.
Since we generate a set of random features $\{ \omegab_k \}_{k\in[N]} \subset \R^d$ and create new features by mapping $\xb$ to 
$$\phi(\xb) = \frac{1}{\sqrt{N}}\left[\exp(i\langle\omegab_1,\xb\rangle), \dots, \exp(i\langle\omegab_N,\xb\rangle) \right]^T \in\C^N,$$
then we could show that $\Tr(\Hb^a) = \E_{\xb \in D_a} \|\phi(\xb)\|^2 = 1$, which is independent of the distribution of input data. Therefore, the excessive risk gap for each group $a$ satisfies $\xi_a=0$ since the trace does not depend on the distribution of input data.

\section{Numerical Experiments}
\label{Sec:numerical}
In this section, we present numerical experiments on the DP random features model in Algorithm \ref{alg:DPRF}. 
We divide this section into two parts: i) we compare our proposed method with the several regression baselines in the literature, and ii) we numerically verify that the DP random feature model has the potential to reduce disparate impact via the notations of excessive risk gap and statistical parity. The code is available on Github repository: \url{https://github.com/liaochunyang/DPRFM}

\subsection{Generalization performance}
In this section, we compare the generalization 
performance of our proposed method with the regression baselines in \cite{Chaudhuri2011DifferentiallyPE} and \cite{Wang2024DPRFM}. The generalization performance is evaluated on a hold-out test set $D_{test}$ of size $|D_{test}|$. We define the test error as
\begin{equation*}
TestError = \frac{1}{|D_{test}|} \sum_{(\xb_i,y_i)\in D_{test}} (y_i - \hat{f}(\xb_i))^2,  
\end{equation*}
where $\hat{f}$ is an approximator. We test our algorithm on both synthetic and real data. 
Throughout the experiments, we consider the Gaussian random feature models where random feature $\omegab\in\R^d$ is drawn from a multivariate normal distribution $\Nc(0,\sigma^2\Ib_d)$. The choice of variance $\sigma^2$ varies among each example and will be stated clearly. \\

\noindent {\bf Differences between regression baselines.} In our numerical experiments, we mainly compare our proposed method with the regression baselines in \cite{Chaudhuri2011DifferentiallyPE} and \cite{Wang2024DPRFM}. 
The difference between our method and that in \cite{Chaudhuri2011DifferentiallyPE} is that different random noises are added to produce private estimators. Our method adds Gaussian random noise, while the existing method in \cite{Chaudhuri2011DifferentiallyPE} adds random noise generated from a probability distribution with density function $\propto \exp(-\xi\|b\|_2)$. 
Our theory suggests that the private model with Gaussian random noise has better generalization performance. The numerical experiments verify our theory empirically. 
In \cite{Wang2024DPRFM}, the authors studied a private random feature model via SGD. However, the non-private solution solved by using SGD may not interpolate the training data due to the restrictions on the learning rate and total number of iterations. 
Our method focuses on the interpolation regime and then produces a private estimator. In our numerical experiments, the methods in \cite{Chaudhuri2011DifferentiallyPE} and \cite{Wang2024DPRFM} will be referred to as {\bf Gamma} and {\bf SGD}, respectively. (The former is called Gamma since the $\ell_2$ norm of additive random noise follows a Gamma distribution.)

The setting of SGD follows the set-up from \cite{Wang2024DPRFM}. Specifically, the initialization will be the zero vector, the batch size will be 1, the learning rate is selected to be $1/m$, and the number of iterations is $m$, where $m$ is the training samples size. \\

\noindent {\bf Solving the min-norm problem.} One of the major contributions of our paper is considering the over-parametrized regime where we design a private random feature model through the min-norm interpolator. 
There are many computational methods to solve the min-norm problem. In our implementation, the min-norm problem is solved (approximated) using the routine numpy.linalg.pinv from the Numpy library, and the randomized Kaczmarz method \cite{thomas2009Kaczmarz}.
The routine numpy.linalg.pinv computes the Moore-Penrose pseudo-inverse of a matrix using the singular value decomposition. 
The randomized Kaczmarz method, which is a special case of SGD, converges linearly to the min-norm solution \cite{Ma2015KZ}. 
We compare the performance of different methods as well, see Table \ref{Table:real}. \\ 

\noindent {\bf Synthetic data.} For synthetic data, each sample $\xb\in\R^d$ is drawn from a multivariate normal distribution $\Nc(0, \Ib_d)$, and its corresponding label $y=f(\xb)$. 
The label $y$ then is normalized by dividing its $\ell_2$ norm. 
We generate 1000 samples for training and 1000 samples (from the same distribution) for testing.
Two test functions $f_1(\xb) = \sqrt{1+\|\xb\|_2} $ and $f_2(\xb) = \sum_{i=1}^d \exp(-|x_i|)$ are considered in our experiments.
To generate the Gaussian random features, the variance $\sigma^2$ is selected to be $\sigma^2=40$ for all experiments.

Figure \ref{Fig:test1} and \ref{Fig:test2} show the test error of our method compared to that of \cite{Chaudhuri2011DifferentiallyPE}, of \cite{Wang2024DPRFM}, and of the non-private random feature model. The test errors are estimated across 10 runs for each number of features $N$ on the test functions $f_1(\xb)$ and $f_2(\xb)$, respectively. 
The figures clearly show that our proposed method has better generalization performance across different number of features $N$ and privacy budget $\epsilon_p$. Moreover, for a fixed privacy budget $\epsilon_p$, increasing the number of random features leads to better generalization for our model and the non-private model, but it is not the case for the models in \cite{Chaudhuri2011DifferentiallyPE, Wang2024DPRFM}. This observation also supports our theoretical results, see the discussion after Lemma \ref{tail_z}. 

\begin{figure}[!htbp]
\centering     
\subfigure[Privacy budget $\epsilon_p=0.5$]{\includegraphics[width=73mm]{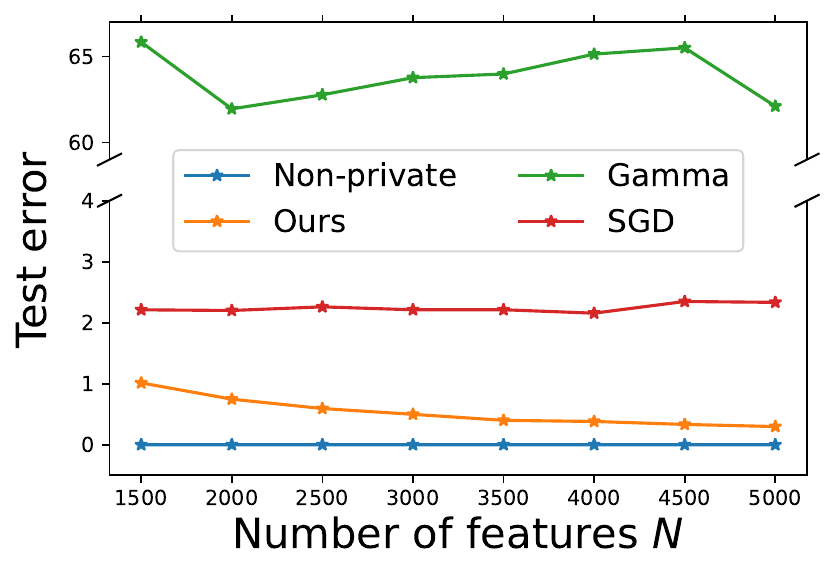}}
\subfigure[Privacy budget $\epsilon_p=1$]{\includegraphics[width=73mm]{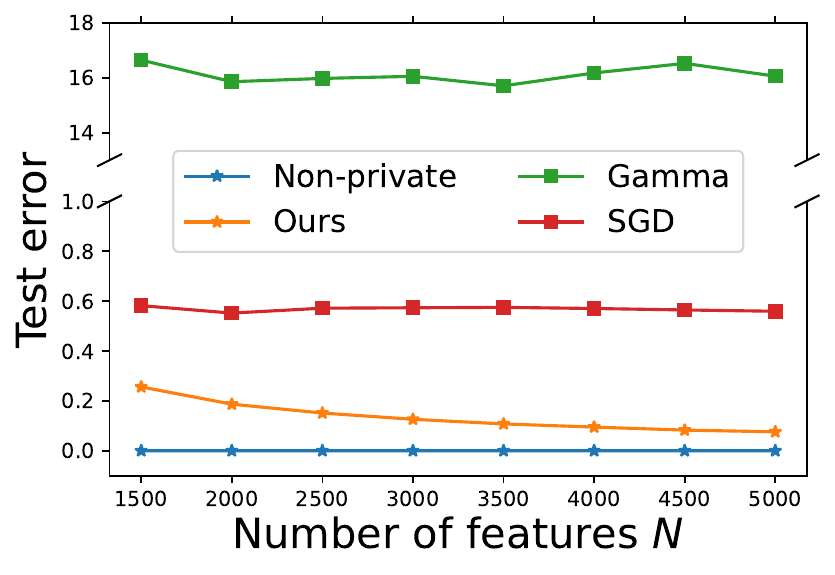}}
\caption{Test errors of non-private random feature model, our proposed model and regression baselines on test function $f_1(\xb)$.}
\label{Fig:test1}
\end{figure}

\begin{figure}[!htbp]
\centering     
\subfigure[Privacy budget $\epsilon_p=0.5$]{\includegraphics[width=73mm]{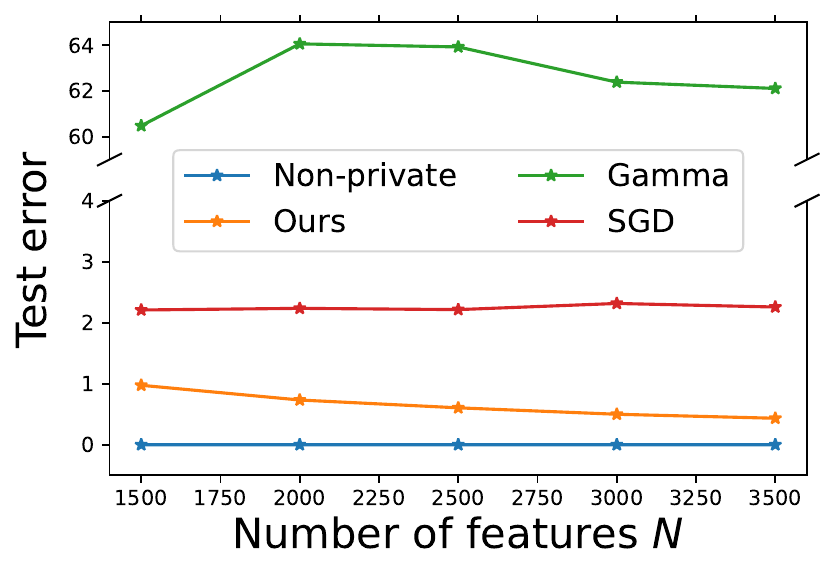}}
\subfigure[Privacy budget $\epsilon_p=1$]{\includegraphics[width=73mm]{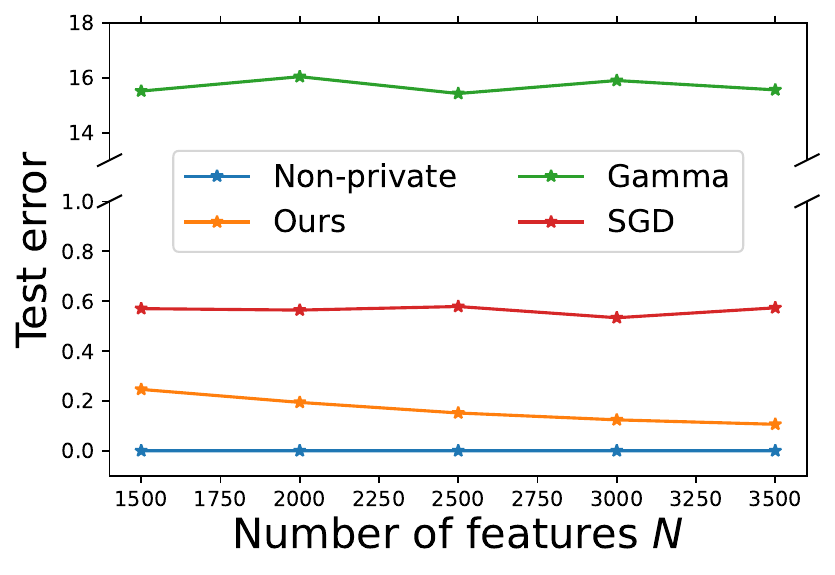}}
\caption{Test errors of non-private random feature model, our proposed model and regression baselines on test function $f_2(\xb)$.} 
\label{Fig:test2}
\end{figure}

Figure \ref{Fig:training} shows how the test error changes as the number of training samples $m$ increases. We set the number of features $N=m+200$ and the privacy budget $\epsilon_p=1$. The test errors are computed on 200 test samples generated from the same distribution as the training samples. We repeat the experiments 10 times and take the average for each point in the figure. 
Our experiments verify the generalization error bound in \cite{Wang2024DPRFM}, which decays in a certain rate depending on the sample size. 
However, increasing sample size may not be possible in practice. Our experiments suggest that our proposed method has better generalization performance than the baseline in \cite{Wang2024DPRFM} when the training samples are scarce. 
Moreover, our proposed method outperforms the regression baseline in \cite{Chaudhuri2011DifferentiallyPE} in all experiments. 

\begin{figure}[!htbp]
\centering     
\subfigure[Test function 1]{\includegraphics[width=73mm]{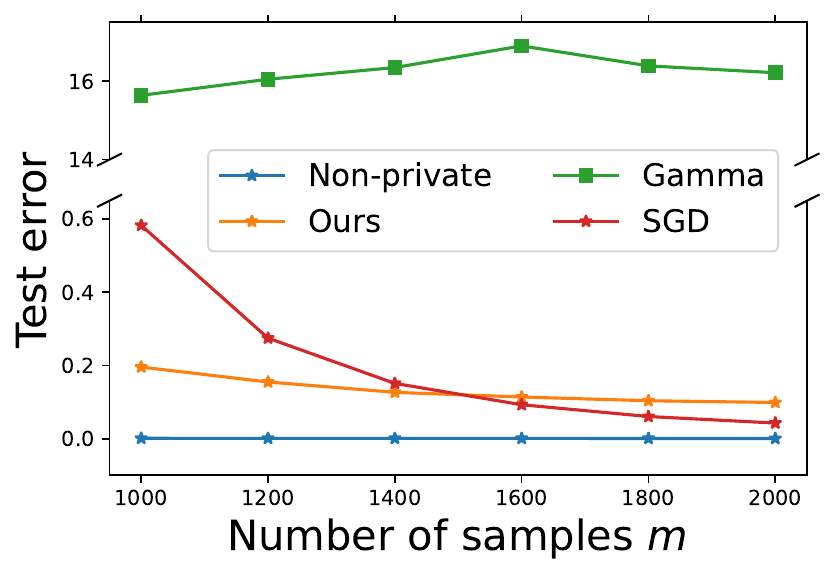}}
\subfigure[Test function 2]{\includegraphics[width=73mm]{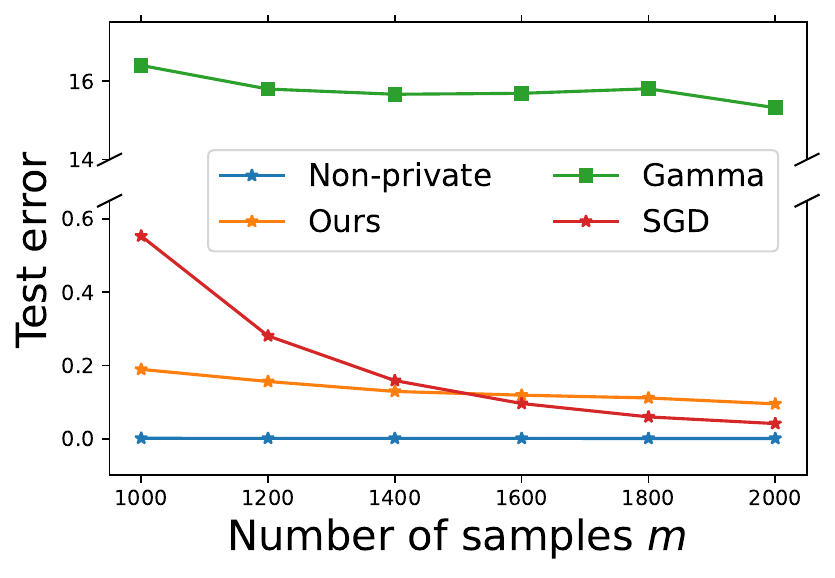}}
\caption{Test errors of non-private random feature model, our proposed model and regression baselines versus the number of training samples $m$ on two test functions.}
\label{Fig:training}
\end{figure}

\noindent {\bf Real data.} For our experiments on real data, we use the \textit{medical} dataset \cite{Lantz2013Medical} and the \textit{wine quality} dataset \cite{wine_quality_186}. 
The medical dataset looks at an individual medical cost prediction task. Each individual has three numeric {\it \{age, BMI, \#children\}} and three categorical features {\it \{sex, smoker, region\}}. 
The dataset also has a real-valued medical {\it charges} column that we use as label. 
We use min-max scaling to normalize the numeric features as well as the label to be within the range $[0,1]$. For any categorical features, we use standard one-hot-encoding. 
The wine quality dataset is used to model wine quality based on physicochemical tests, which contains 11 numeric features and 1 label column named {\it quality}. There are no categorical features in this dataset. 
Similarly, we use min-max scaling to normalize the numeric features as well as the label.

\begin{figure}[!htbp]
\centering     
\subfigure{\includegraphics[width=73mm]{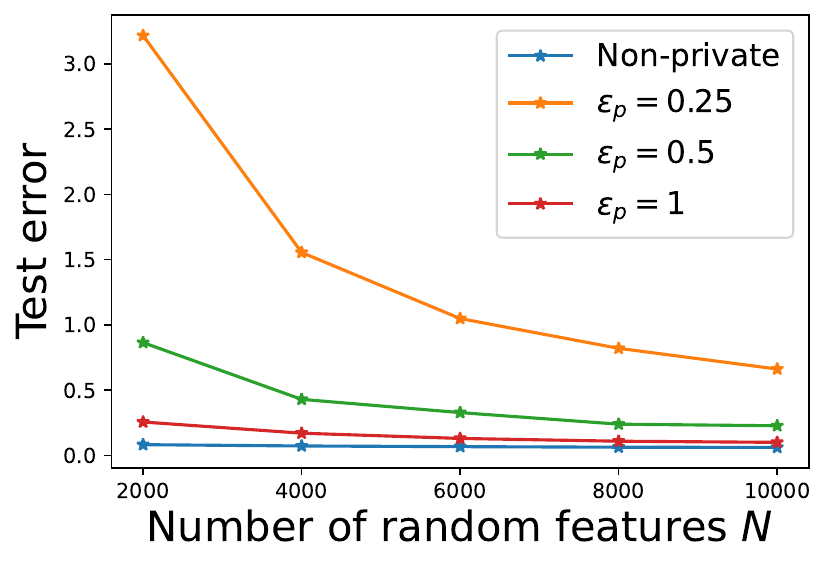}}
\subfigure{\includegraphics[width=73mm]{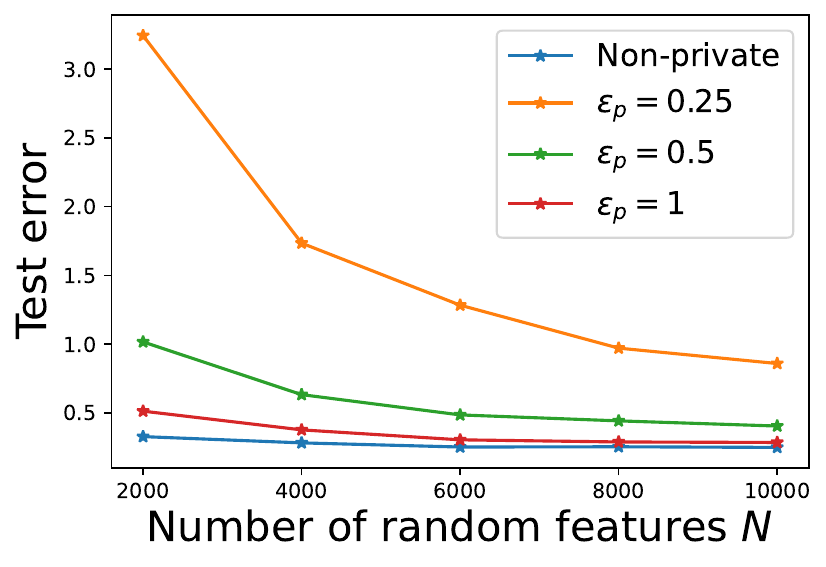}}
\caption{Test performances on real datasets of our proposed model cross different privacy budgets $\epsilon_p$. On the left: Medical cost dataset. On the right: Wine quality dataset.}
\label{Fig:real}
\end{figure}

In Figure \ref{Fig:real}, we show that the test errors of our private and of non-private random feature models on two real datasets. 
We generate the random features with the same variance $\sigma^2=40$ for all experiments.
Each point in the figure represents the average of 10 runs over the random features and private estimators given a fixed number of random features $N$ and privacy budget $\epsilon_p$. 
For each fixed number of random features $N$, We observed that the test error decays as the privacy budget increases. 
Moreover, increasing $N$ also decreases test error for fixed each privacy budget. 

{\color{black} We also compare our model with the regression baseline proposed in \cite{Wang2024DPRFM} in terms of test error. 
To implement SGD, we take the batch size to be 1, number of iterations to be $m$, and learning rate $\lambda = \{1/m, 1/(2m), 1/(8m)\}$ as suggested in \cite{Wang2024DPRFM}, where $m$ is the sample size.
The min-norm interpolation problem is solved by using the randomized Kaczmarz algorithm, which is a special case of SGD.
The number of iterations is selected to be $m$ as well.
The results are shown in Figure \ref{Fig:real1_error} (Medical cost dataset) and Figure \ref{Fig:real2_error} (Wine quality dataset). We first compare the performances of non-private random feature models produced by solving the min-norm interpolation problem and by using SGD.
Our results indicate that the min-norm solution can generalize well and achieve better test performance than models trained by SGD. We also observe that SGD (with other learning rates) does not guarantee the convergence to the min-norm solution. 
On the comparisons of private random feature models, our experiments show that our proposed method has better test performance than the regression baseline.} 

\begin{figure}[!htbp]
\centering     
\subfigure[Test error comparisons on the non-private random feature models.]{\includegraphics[width=73mm]{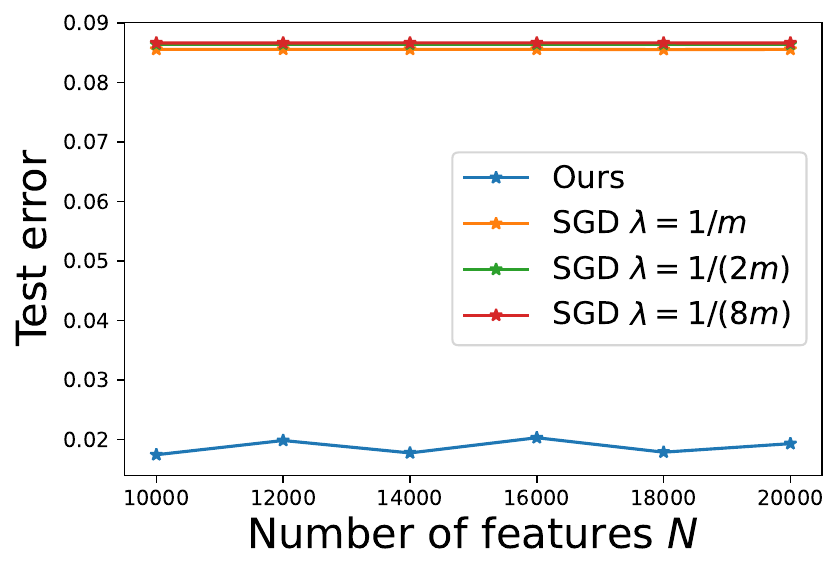}}
\subfigure[Test error comparisons on the private random feature models.]{\includegraphics[width=73mm]{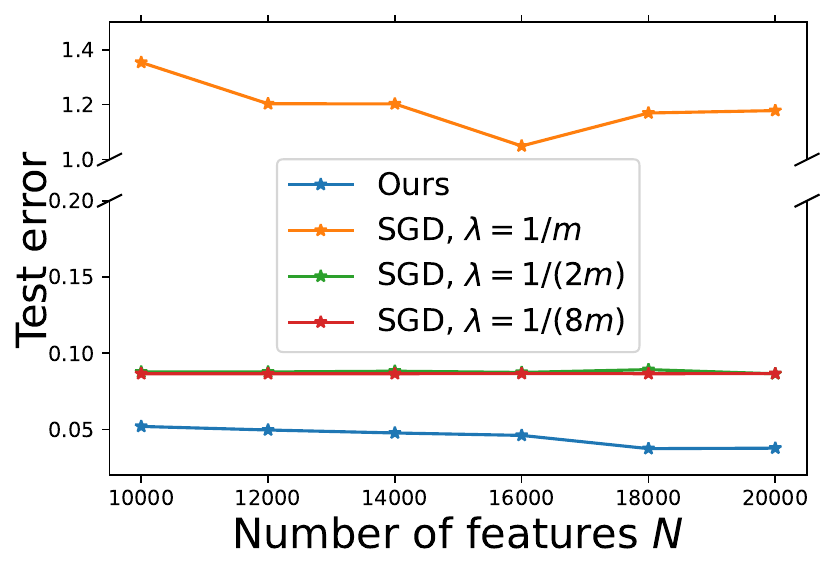}}
\caption{Medical cost dataset: we compare test errors of non-private random feature models (on the left) and corresponding private random feature models (on the right). Differentlearning rates $\lambda$ are considered when using SGD. Gaussian random features are generated from normal distribution with mean 0 and variance $\sigma^2=2$. The privacy budget $\epsilon_p$ is 0.5 and the privacy parameter $\delta_p=10^{-5}$. }
\label{Fig:real1_error}
\end{figure}

\begin{figure}[!htbp]
\centering     
\subfigure[Test error comparisons on the non-private random feature models.]{\includegraphics[width=73mm]{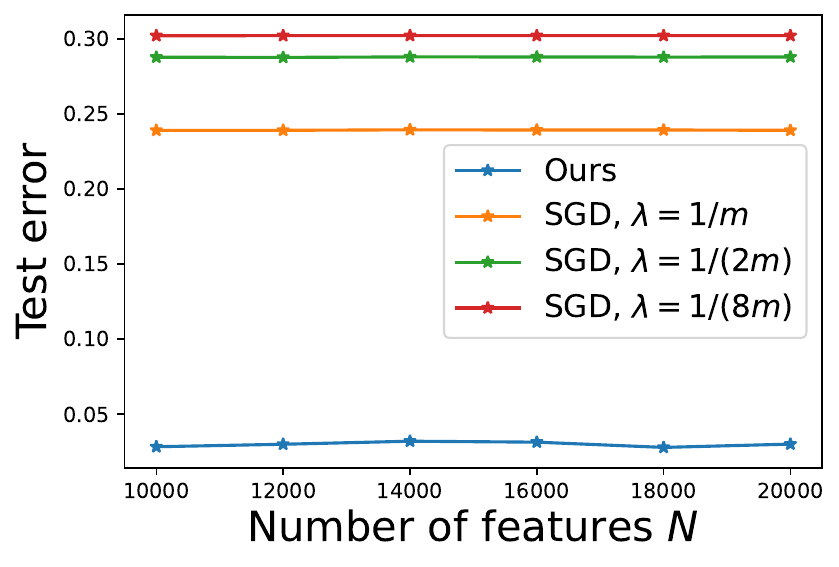}}
\subfigure[Test error comparisons on the private random feature models.]{\includegraphics[width=73mm]{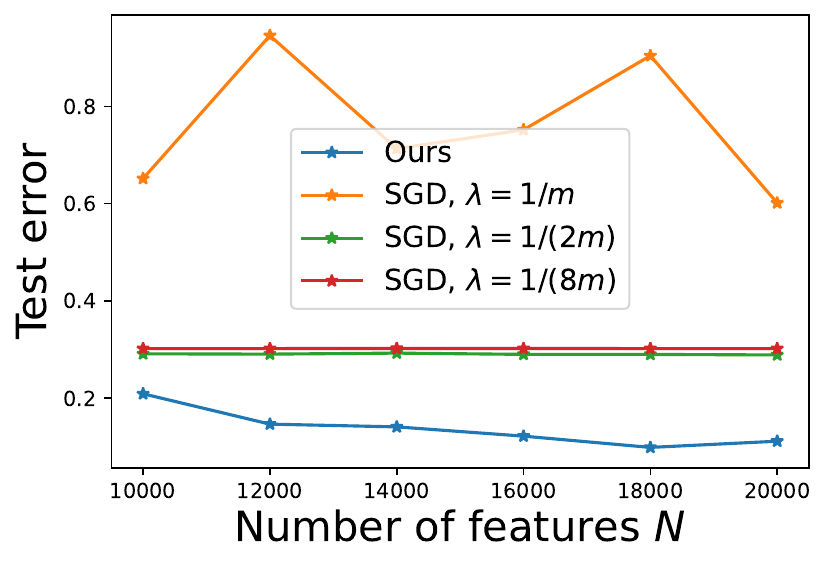}}
\caption{Wine quality dataset: we compare test errors of non-private random feature models (on the left) and corresponding private random feature models (on the right). Different learning rates $\lambda$ are considered when using SGD. Gaussian random features are generated from normal distribution with mean 0 and variance $\sigma^2=2$. The privacy budget $\epsilon_p$ is 0.5 and the privacy parameter $\delta_p=10^{-5}$.}
\label{Fig:real2_error}
\end{figure}

{\color{black} Finding the non-private min-norm solution is the most time consuming step in our method, especially when the number of random features $N$ is large. 
In our experiments, we use the randomized Kaczmarz method to compute (approximate) the non-private min-norm solution, and then produce the private random feature model by adding Gaussian random noise. 
There are many methods or algorithms to find the min-norm solution. 
For example, we can use the routine numpy.linalg.pinv in the Numpy library, which approximates the Moore-Penrose pseudo-inverse of a matrix using singular value decomposition, and then approximates the min-norm solution.
We use two methods to compute(approximate) the non-private min-norm solution. Then, we produce the private models following Algorithm \ref{alg:DPRF} with the same Gaussian random noise (random noise with the same distribution).
We compare those two methods in terms of test error and training time. 
Table \ref{Table:real} reports the test errors and training times on different benchmarks. 
Our results suggest that randomized Kaczmarz method does converge to the min-norm solution and the choices of methods for computing non-private min-norm solution do not affect the test errors. However, the training time can be reduced by applying the randomized Kaczmarz method, especially when the number of random features is large. } 

\begin{table}[!htbp]
\centering
\begin{tabular}{|c|c|c|c|c|c|c|c|}
\hline
Dataset & Method & $N$  & 2000 & 4000 & 6000 & 8000 & 10000  \\ \hline
\multirow{4}{*}{Medical cost} &  \multirow{2}{*}{pinv} & test error & 0.39 & 0.19 & 0.14 & 0.12 & 0.11  \\ \cline{3-8}  
& & time & 0.61 & 0.96 & 1.44 & 1.92 & 2.44  \\ \cline{2-8}
& \multirow{2}{*}{Kaczmarz} & test error & 0.29 & 0.17 & 0.14 & 0.12 & 0.11  \\ \cline{3-8}  
& & time & 0.02 & 0.02 & 0.04 & 0.07 & 0.08  \\ \hline
\multirow{4}{*}{Wine quality} &  \multirow{2}{*}{pinv} & test error & 0.79 & 0.41 & 0.33 & 0.30 & 0.31  \\ \cline{3-8}  
& & time & 0.91 & 1.41 & 2.07 & 2.91 & 3.43  \\ \cline{2-8}
& \multirow{2}{*}{Kaczmarz} & test error & 0.52 & 0.39 & 0.35 & 0.32 & 0.32  \\ \cline{3-8}  
& & time & 0.02 & 0.03 & 0.07 & 0.07 & 0.10  \\ \hline
\end{tabular}
\caption{Summary of numerical results: we report test errors and training times for different methods on computing the min-norm solution. We test on medical cost and wine quality datasets. We set privacy budget $\epsilon_p=1$ and privacy parameter $\delta_p=10^{-5}$.}
\label{Table:real}
\end{table}

\subsection{Excessive Risk Gap} 
\label{Sec:ERG}
In this section, we numerically show that our proposed DP random features model achieves smaller excessive risk gap resulting in fairer results.
We use the \textit{medical cost} and \textit{wine quality} datasets for this section and compare our proposed DP random features model with the private regularized linear regression model proposed in \cite{Chaudhuri2011DifferentiallyPE, acharya2024personalized}.

\noindent {\bf Data processing.} We use standard one-hot-encoding for categorical features and {\color{black} normalize labels $y$ by dividing its $\ell_2$ norm.} 
For the random feature model, we do not normalize any numerical features since we want to make sure that the changes in the excessive risk gap do not depend on data normalization. 
For the regularized linear regression model, we normalize the input values for each group independently.
In the wine quality experiment, group 1 and 2 represent red wine and white wine, respectively. We randomly sample 1000 data points from each group for training and 500 samples from each group for testing.
In the medical cost experiment, we first use categorical feature {\it sex} to create two different groups, i.e. group 1 is male group and group 2 is female group. 
In the second experiment, we divide the dataset into two groups according to the categorical feature {\it smoker}, where group 1 is smoker group and group 2 is non-smoker group. In those examples, 90\% of the data are randomly sampled from each group for training and the remaining 10\% of data are for testing.

\noindent {\bf Experiment setting.} In all experiments, the privacy budget $\epsilon_p$ varies from 0.05 to 0.3 and the parameter $\delta_p$ is chosen to be $\delta_p=10^{-5}$. 
The excessive risk gap 
is approximated by the average of 100 repetitions on the output perturbation. For the random feature model, we generate $N=4000$ random features $\omegab\in\R^d$ from a multivariate Gaussian distribution $\Nc(0,\sigma^2\Ib_d)$ with $\sigma^2 = 20$ for the wine quality dataset, and with $\sigma^2 = 2\times 10^{-5}$ for the medical cost dataset. The hyperparameter $\sigma^2$ will affect the model performance and should be chosen carefully. {\color{black} Usually, one can use cross-validation to select a good $\sigma^2$.} 
 We follow Algorithm \ref{alg:DPRF} to generate the private random feature model with hyperparameter $\eta=3/8$ which means the $\ell_2$ sensitivity is $\frac{4}{\sqrt{N}}$. To generate the private regularized linear regression model, we also adopt the Gaussian mechanism. Following the calculation in \cite{Chaudhuri2011DifferentiallyPE, acharya2024personalized}, the $\ell_2$ sensitivity of the coefficient vector is $\frac{2}{m\lambda}$ where $m$ is training size and $\lambda$ is the regularization parameter. 
{\color{black} We select $\lambda = \sqrt{N}/(2m)$ such that the added Gaussian noises for both models have the same distribution, i.e. Gaussian distribution with the mean 0 and the same variance.}

\begin{figure}[!htbp]
\centering     
\subfigure[Sensitive feature "sex"]{\includegraphics[width=70mm]{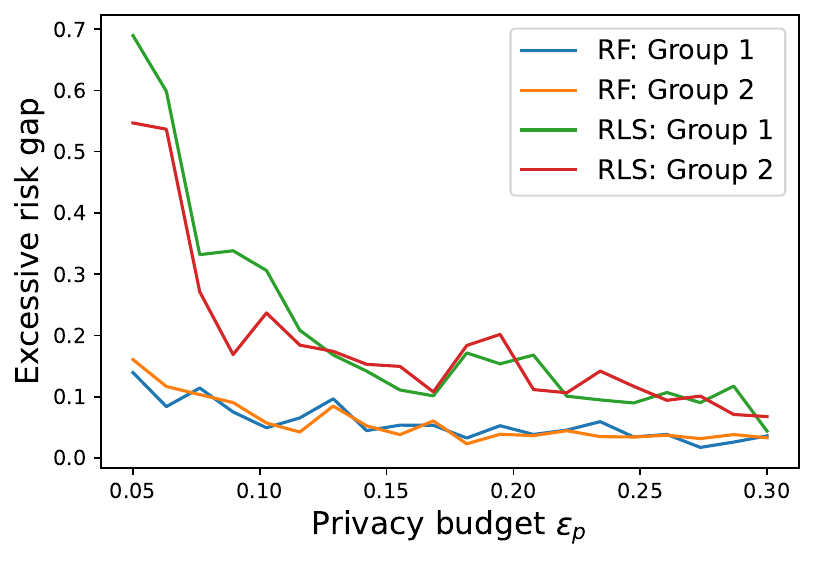}}
\subfigure[Sensitive feature "smoker"]{\includegraphics[width=70mm]{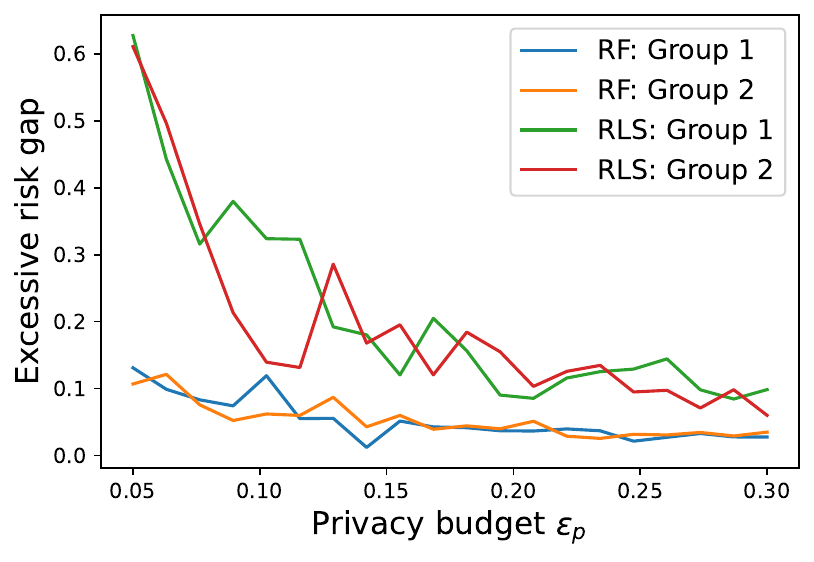}}
\caption{Medical cost dataset: Training excessive risk gaps for each group versus privacy budget $\epsilon_p$ for private random features model and regularized linear regression model.} 
\label{Fig:fairness1}
\end{figure}

\begin{figure}[!htbp]
\centering     
\includegraphics[width=70mm]{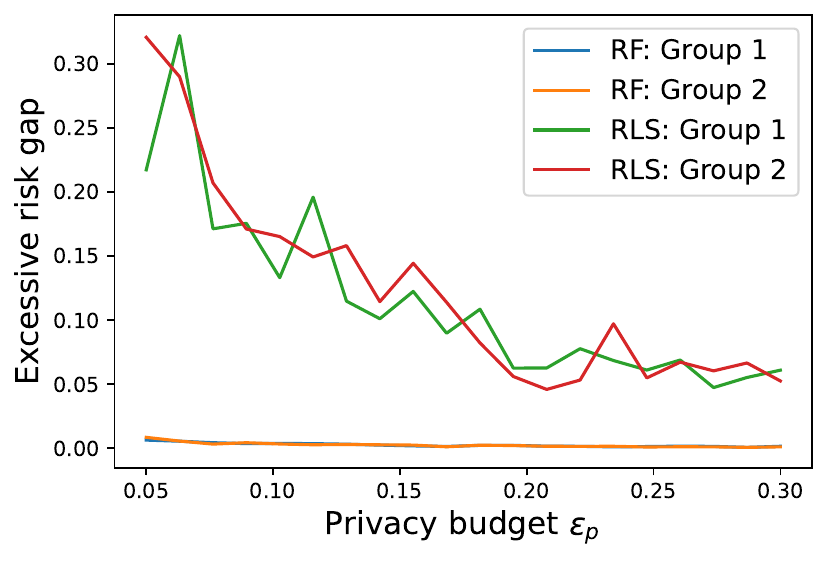}
\caption{Wine quality dataset: Training excessive risk gaps for each group versus privacy budget $\epsilon_p$ for private random features model and regularized linear regression model.}
\label{Fig:fairness2}
\end{figure}

Figure \ref{Fig:fairness1} and \ref{Fig:fairness2} show the relations between the excessive risk gaps for each group and the machine learning models while varying the privacy budget $\epsilon_p$.
Each data point represents the average of 10 runs of a train-test split. 
In all experiments, the excessive risk gaps for each group of the random feature model are smaller than excessive risk gaps of the regularized linear regression model (provided the same privacy budget). This means that the random feature model does not lead to utility loss. 
Moreover, the excessive risk gaps across different groups are similar for the random feature model, especially when the privacy budget $\epsilon_p$ is small, which indicates that better fairness 
may be achieved by using the private random feature model. In addition, we observe that the excessive risk gap decreases as the privacy budget $\epsilon_p$ increases. The same phenomenon has been observed in \cite{tran2021differentially}. {\color{black} This is likely due to the fact that larger $\epsilon_p$ values require smaller $\sigma$ values, and thus, as shown in equation \eqref{excessive_risk_gap}, the excessive risk gap decreases for a fixed Hessian trace.} 
In summary, we believe that better fairness may be achieved by using the random features model and it is thus superior to feature normalization.

\subsection{Statistical Parity}

In this section, we adopt the notation of statistical parity as the fairness metric and compare the output distributions of our proposed DP random features model and of the private regularized linear regression model.
We use the medical cost and wine quality datasets for numerical experiments as well.

\noindent {\bf Experiment setup.} We keep similar data processing procedures and experiment settings as in Section \ref{Sec:ERG}. Instead of downsampling and splitting data samples to training and test sets, we use the whole dataset to train the models. Moreover, we fix the privacy budget $\epsilon_p=0.5$ for the medical cost example and $\epsilon_p=0.05$ for the wine quality example.

\begin{figure}[!htbp]
\centering     
\subfigure{\includegraphics[width=36mm]{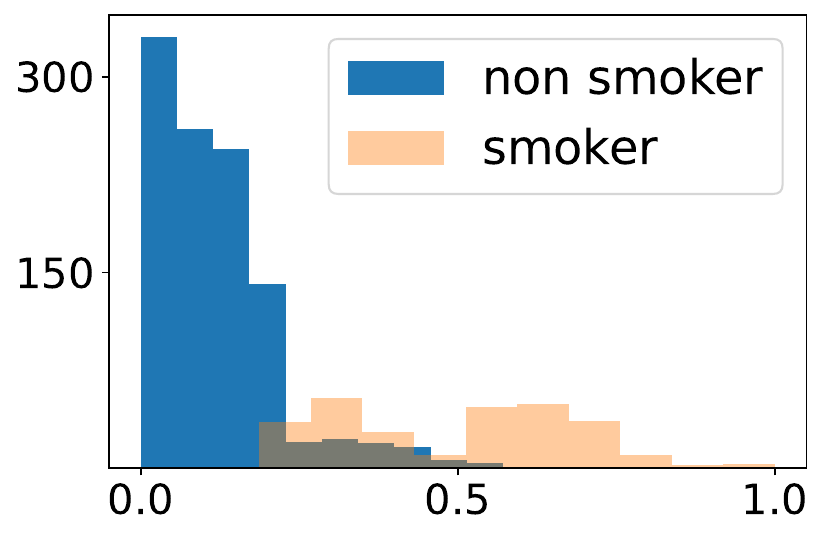}}
\subfigure{\includegraphics[width=36mm]{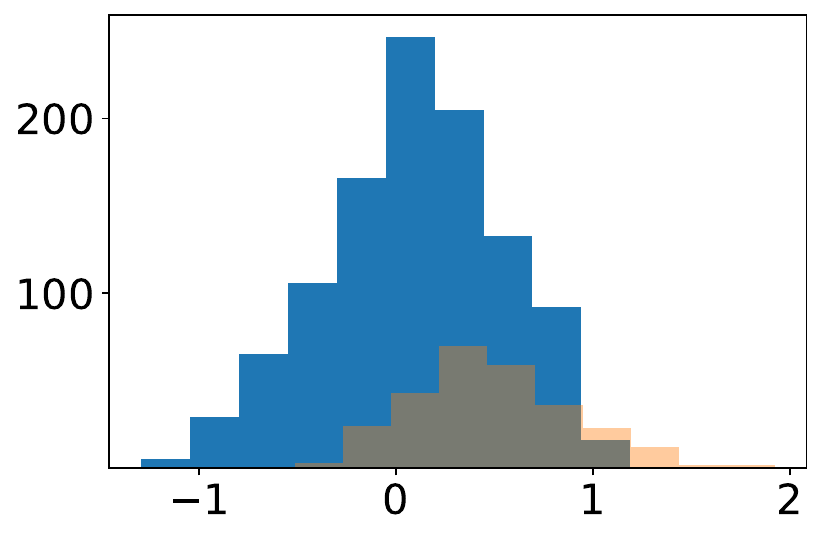}}
\subfigure{\includegraphics[width=36mm]{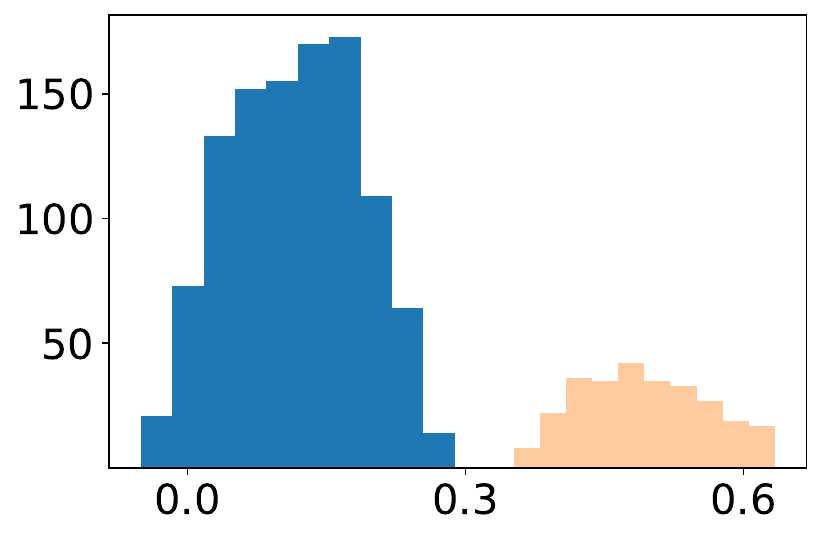}}
\subfigure{\includegraphics[width=36mm]{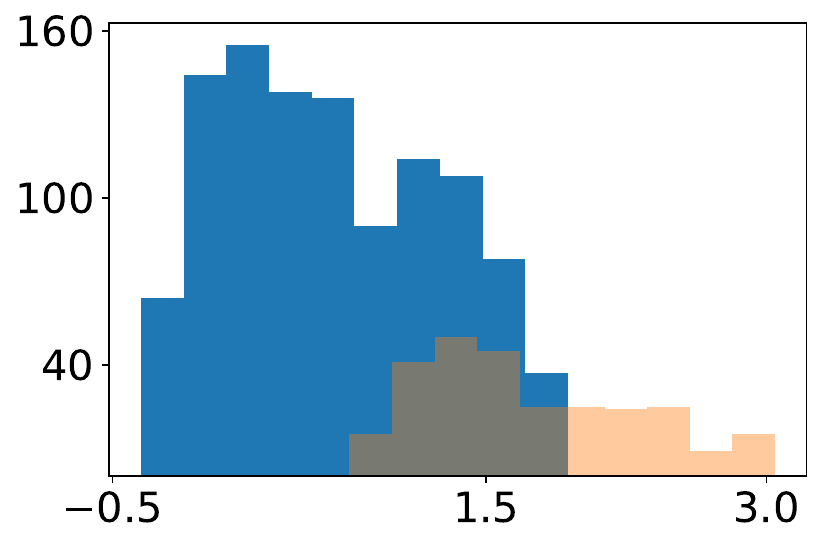}}
\caption{Medical cost with sensitive feature "smoker": distributions of targets and of the outputs of regressors. From left to right: targets, DP random features model, regularized least-squares, DP regularized least-squared.} 
\label{Fig:SP}
\end{figure}

Our results are presented in Figure \ref{Fig:SP} and Table \ref{Table:SP}. In Figure \ref{Fig:SP}, we depict the distributions of targets in the dataset, outputs of DP random features model, outputs of regularized least-squares method, and DP regularized least-squares method. 
The result shows that the outputs of the DP random features model conditioned on different sensitive groups have similar distributions.
Furthermore, we report estimated statistical parity scores, as defined in Definition \ref{def:SP}. To estimate the Kolmogorov-Smirnov distance of output distributions conditioned on different groups, we take a grid of resolution 500 over the domain and estimate the probability that outputs are less than the grid points. Then we take the maximum to get the estimated Kolmogorov-Smirnov distance.
The experiments are repeated 20 times over the sampling of random features and the DP mechanism.
From Table \ref{Table:SP}, we observe that the DP random features model can reduce the statistical parity score even if the score is high in the original targets. 
Notice that our method does not require post-processing to have the same output distributions for different groups.
In summary, we believe that better statistical parity may be achieved by using the random features model. 

\begin{table}[!htbp]
\centering
\begin{tabular}{|c|c|c|c|c|}
\hline
Dataset & Targets & DP-RF & RLS & DP-RLS   \\ \hline
Medical Cost (``sex'') & 0.06  & $0.117\pm0.047$ & 0.074 & $0.427\pm0.233$ \\\hline
Medical Cost (``smoker'') & $0.873$  & $0.410 \pm 0.098$ & 0.999 & $0.516 \pm 0.302$\\\hline
Wine Quality & 0.130 & $0.028\pm0.009$ & 0.279 & $0.312\pm0.239$ \\\hline
\end{tabular}
\caption{Summary of estimated statistical parity scores.}
\label{Table:SP}
\end{table}

\section{Discussion}
\label{Sec:discussion}
In this paper, we established theoretical guarantees for the differentially private random feature model. {\color{black} The random feature model is a widely used surrogate model for differentially private learning with kernels.} Apart from the current literature where people focused on the regularization regime, we focused on the overparametrized regime and studied the min-norm interpolation estimator. Our setting is closer to the modern deep learning setting where the number of trainable parameters is much larger than the number of samples and explicit regularization is not added. Our privacy guarantee follows the standard Gaussian mechanism and the derivation of the generalization error bound relies on tools from random matrix theory. We also show that the proposed DP random features model seems to not suffer disparate impact. Additionally, numerical experiments showed that our method has better generalization performance under the same privacy budget compared with the current differentially private random feature method, and has the potential to reduce disparate impact.

We conclude with some directions for future work. First, it would be interesting to study differentially private shallow neural networks and kernel methods through the random features model. 
Moreover, empirical privacy evaluation on how random features models leak information about the individual data records deserves future investigation.
Finally, we can consider other fairness metrics, such as Equalized Odds. For the fairness results, our focus was on the computational side, and we expect theoretical guarantees on the trade-offs between privacy, fairness, and utility to be interesting future work.

\section*{Appendix}

In the appendix, we give a full proof of Theorem \ref{error_non_private}, which bounds the generalization error of non-private random features model. Lemmas \ref{term1_lem1} and \ref{term1_lem2} are used to bound approximation error, which addresses the difference between target function and the "best" random features model.

\begin{lemma}
\label{term1_lem1}
Let $f$ be the target function belonging to $\Fc(\rho)$ and $f^*$ be defined as \eqref{rf_best}, respectively. Then for all $T>0$, we have 
\begin{equation}
\left| f(\xb) - E_{\omegab} f^*(\xb) \right|^2 \leq \frac{\left(\E_{\omegab}[\alpha(\omegab)^2]\right)^2}{T^2} = \frac{\|f\|_{\rho}^4}{T^2}.
\end{equation}
Furthermore, with $T$ being selected as $T=\sqrt{N}\|f\|_\rho^2$, then it holds that 
\begin{equation*}
\left| f(\xb) - E_{\omegab} f^*(\xb) \right|^2 \leq \frac{\|f\|_\rho^2}{N} 
\end{equation*}
for all $\xb\in\R^d$.
\end{lemma}
\begin{proof}
The second equality follows the definition of $\alpha(\omegab)$ and $\|f\|_\rho$.
The inequality follows from 
\begin{equation}
\begin{aligned}
|f(\xb) - E_{\omegab} f^*(\xb)|^2 =& | \E_{\omegab} \left[ \alpha_{>T}(\omegab)\exp(i\langle \omegab,\xb\rangle) \right] |^2 \\
\leq& \E_{\omegab} \left[\alpha(\omegab)\right]^2 \E_{\omegab} \left[ \indicator_{\left|\alpha(\omegab)\right|> T}\exp(i\langle \omegab,\xb\rangle)\right]^2  \\
=& \E_{\omegab} \left[\alpha(\omegab)\right]^2 \Pbb(\alpha(\omegab)^2>T^2) \\
\leq& \frac{\left(\E_{\omegab}[\alpha(\omegab)^2]\right)^2}{T^2}
\end{aligned}
\end{equation}
where we use the Cauchy-Schwarz inequality in the second line and the Markov's inequality in the last line. 
\end{proof}

\begin{lemma}
\label{term1_lem2}
Let $f^*$ be defined as \eqref{rf_best}. For all $T>0$, the following inequality holds with probability at least $1-\delta$,
\begin{equation}
\left|f^*(\xb) - E_{\omegab} f^*(\xb) \right|^2 \leq \frac{32T^2\log^2(2/\delta)}{N^2} + \frac{4\|f\|^2_\rho\log(2/\delta)}{N}.
\end{equation}
Furthermore, with $T$ being selected as $T=\sqrt{N}\|f\|_\rho^2$, then it holds with probability at least $1-\delta$ of drawing random features $\omegab$ that
\begin{equation*}
\left|f^*(\xb) - E_{\omegab} f^*(\xb) \right|^2 \leq \frac{32\log^2(2/\delta)\|f\|_\rho^2}{N} + \frac{4\|f\|_\rho^2\log(2/\delta)}{N}.
\end{equation*}
\end{lemma}
\begin{proof}
For each $\xb\in D$, we define random variable $Z(\omegab) = \alpha_{\leq T}(\omegab)\exp(i\langle \omegab,\xb\rangle)$ and let $Z_1, \dots, Z_N$ be $N$ i.i.d copies of $Z$ defined by $Z_k = Z(\omegab_k)$ for each $k\in[N]$.
By boundness of $\alpha_{\leq T}(\omegab)$, we have an upper bound $|Z_k| \leq T$ for any $k\in[N]$. The variance of $Z$ is bounded above as 
$$
\sigma^2 := \E_{\omegab} |Z-\E_{\omegab} Z|^2 \leq \E_{\omegab} |Z|^2 \leq \E_{\omegab} [\alpha(\omegab)^2] = \|f\|^2_\rho. 
$$
By Lemma A.2 and Theorem A.1 in \cite{lanthaler2023error}, it holds with probability at least $1-\delta$ that
\begin{equation*}
\left|f^*(\xb) - E_{\omegab} f^*(\xb) \right| = \left|\frac{1}{N}\sum_{k=1}^N Z_k - \E_{\omegab} Z \right| \leq \frac{4T\log(2/\delta)}{N} + \sqrt{\frac{2\|f\|^2_\rho\log(2/\delta)}{N}}.
\end{equation*}
Taking the square for both sides and using the inequality $(a+b)^2\leq 2a^2+2b^2$ give the desired result.
\end{proof}

In the next lemma, we show the decay rate of the coefficients $\cb^*$ in the "best" random features model.

\begin{lemma}
\label{decay_c}
Let $f^*$ be defined as \eqref{rf_best} and $\cb^*$ be the corresponding coefficient vector. For some $\delta\in(0,1)$ and any $T>0$, it holds with probability at least $1-\delta$ that
\begin{equation*}
\|\cb^*\|_2^2 \leq \frac{1}{N} \left( \|f\|_\rho^2 + \frac{4T^2\log(2/\delta)}{N} + \sqrt{\frac{2T^2\|f\|_\rho^2\log(2/\delta)}{N}} \right).
\end{equation*}
Furthermore, with $T$ being selected as $T=\sqrt{N}\|f\|_\rho^2$, it holds with probability at least $1-\delta$ that 
\begin{equation*}
\|\cb^*\|_2^2 \leq \frac{12\log(2/\delta)\|f\|_\rho^2}{N}.
\end{equation*}
\end{lemma}
\begin{proof}
By the definition of $c_k^*$, we have
\begin{equation}
\|\cb^*\|_2^2 = \sum_{k=1}^N |c_k^*|^2 = \frac{1}{N^2} \sum_{k=1}^N |\alpha_{\leq T}(\omegab_k)|^2.     
\end{equation}
Define random variable $Z(\omegab) = |\alpha_{\leq T}(\omega)|^2$ and let $Z_1,\dots, Z_N$ be $N$ i.i.d copies of $Z$ defined as $Z_k = |\alpha_{\leq T}(\omegab_k)|^2$ for each $k\in[N]$.
By the boundedness of $\alpha_{\leq T}(\omegab)$, we have an upper bound $|Z_k|\leq T^2$ for each $k\in[N]$. The variance of $Z$ is bounded above as 
\begin{equation*}
\sigma^2 := \E_{\omegab} |Z-\E_{\omegab} Z|^2 \leq \E_{\omegab} |Z|^2 \leq T^2\E_{\omegab} \left[ |\alpha(\omegab)|^2 \right] = T^2\|f\|^2_\rho.
\end{equation*}
By Lemma A.2 and Theorem A.1 in \cite{lanthaler2023error}, it holds with probability at least $1-\delta$ that
\begin{equation*}
\left| \frac{1}{N} \sum_{k=1}^N Z_k - \E_{\omegab} Z \right| \leq \frac{4T^2\log(2/\delta)}{N} + \sqrt{\frac{2T^2\|f\|_\rho^2\log(2/\delta)}{N}}.
\end{equation*}
Then, we have
\begin{equation*}
\|\cb^*\|_2^2 = \frac{1}{N} \left( \frac{1}{N}\sum_{k=1}^N |\alpha_{\leq T}(\omegab_k)|^2 \right) \leq \frac{1}{N} \left( \|f\|_\rho^2 + \frac{4T^2\log(2/\delta)}{N} + \sqrt{\frac{2T^2\|f\|_\rho^2\log(2/\delta)}{N}} \right).
\end{equation*}
Selecting $T=\sqrt{N}\|f\|_\rho$ leads to
\begin{equation*}
\|\cb^*\|_2^2 \leq \frac{\|f\|_\rho^2}{N} \left( 1 + 4\log(2/\delta) + \sqrt{2\log(2/\delta)} \right) \leq \frac{12\log(2/\delta)\|f\|_\rho^2}{N}.
\end{equation*}
\end{proof}

With the above three lemmas at hand, we are ready to present the proof of the generalization error bound of non-private random features model.

\begin{proof}[Proof of Theorem \ref{error_non_private}]
We decompose the generalization error for non-private random feature model $f^\sharp$ into 
\begin{equation*}
\|f-f^\sharp\|_{L^2(\mu)} \leq \| f - f^* \|_{L^2(\mu)} + \|f^* - f^\sharp\|_{L^2(\mu)}.
\end{equation*}
We first give a bound of $\| f - f^* \|_{L^2(\mu)}$. As a consequence of Lemma \ref{term1_lem1} and \ref{term1_lem2}, we can conclude that
\begin{equation}
\label{term1_x}
|f(\xb) - f^*(\xb)|^2 \leq \frac{2\|f\|_\rho^2}{N} + \frac{64\log^2(2/\delta)\|f\|_\rho^2}{N} + \frac{8\|f\|_\rho^2\log(2/\delta)}{N} \leq \frac{192\log^2(2/\delta)\|f\|_\rho^2}{N}
\end{equation}
for each $\xb\in\R^d$ with probability at least $1-\delta$, and hence it holds with probability at least $1-\delta$ that
\begin{equation}
\label{term1}
\|f-f^*\|_{L^2(\mu)} = \left( \int_{\R^d} |f(\xb) - f^*(\xb)|^2 d\mu(\xb) \right)^{1/2}
\leq \frac{14\log(2/\delta)}{\sqrt{N}}\|f\|_\rho.
\end{equation}

To bound the second term, we will use McDiarmid's inequality. Draw i.i.d samples $\Uc = \{\ub_j\}_{j\in[m]}$ from the density $\mu$, and define random variable
\begin{equation*}
\begin{aligned}
v(\ub_1,\dots,\ub_m) &:= \|f^* - f^\sharp\|^2_{L^2(\mu)} - \frac{1}{m}\sum_{j=1}^m \left|f^*(\ub_j) - f^\sharp(\ub_j)\right|^2 \\
&= \frac{1}{m} \E_{\Uc} \left[ \sum_{j=1}^m \left|f^*(\ub_j) - f^\sharp(\ub_j)\right|^2 \right] - \frac{1}{m}\sum_{j=1}^m \left|f^*(\ub_j) - f^\sharp(\ub_j)\right|^2.
\end{aligned}
\end{equation*}
Thus, $v$ has mean zero, i.e. $\E v =0$. The points $\ub_j\in\Uc$ are i.i.d samples, independent of the training samples $\xb_j$ and hence independent of $\cb^*$ and $\cb^\sharp$. To apply McDiarmid's inequality, we first show that $v$ is stable under a perturbation of any one of its coordinates. Perturbing the $k$-th random variable $\ub_k$ to $\Tilde{\ub}_k$ leads to
\begin{equation*}
\begin{aligned}
& |v(\ub_1,\dots,\ub_k,\dots,\ub_m) - v(\ub_1,\dots,\Tilde{\ub}_k,\dots,\ub_m) | \\
=& \frac{1}{m} \left| |f^*(\ub_k) - f^\sharp(\ub_k)|^2 - |f^*(\Tilde{\ub}_k) f^\sharp(\Tilde{\ub}_k)|^2 \right| \\
\leq& \frac{2N}{m} \|\cb^* - \cb^\sharp\|_2^2 := \Delta_v,
\end{aligned}
\end{equation*}
where the third line holds by the Cauchy-Schwarz inequality and 
\begin{equation*}
|f^*(\ub_k) - f^\sharp(\ub_k)|^2 = \left|\sum_{k=1}^N (c_k^* - c_k^\sharp) \exp(i\langle \omegab_k, \ub_k\rangle)\right|^2 \leq N\|\cb^* - \cb^\sharp\|_2^2.
\end{equation*}
Next, we apply McDiarmid's inequality $\Pbb(v \geq t) \leq \exp\left(-\frac{2t^2}{m\Delta_v^2}\right)$, by selecting
\begin{equation*}
t = \frac{\sqrt{2\log\left(\frac{1}{\delta}\right)}N}{\sqrt{m}}\|\cb^* - \cb^\sharp\|_2^2,
\end{equation*}
which leads to
\begin{equation*}
\|f^* - f^\sharp\|^2_{L^2(\mu)} \leq \frac{1}{m}\sum_{j=1}^m \left|f^*(\ub_j) - f^\sharp(\ub_j)\right|^2 +\frac{\sqrt{2\log\left(\frac{1}{\delta}\right)}N}{\sqrt{m}}\|\cb^* - \cb^\sharp\|_2^2
\end{equation*}
with probability at least $1-\delta$ to the draw of $\Uc$. 
Define the random feature matrix $\Tilde{\Ab}$ element-wise by $\Tilde{\Ab}_{j,k} = \exp(i\langle \omegab_k,\ub_j\rangle)$, we have
\begin{equation*}
\begin{aligned}
& \frac{1}{m}\sum_{j=1}^m \left|f^*(\ub_j) - f^\sharp(\ub_j)\right|^2 = \frac{1}{m} \left\|\Tilde{\Ab}(\cb^* - \cb^\sharp)\right\|_2^2 \\
\leq& \frac{1}{m}\|\Tilde{\Ab}\|_2^2 \|\cb^* - \cb^\sharp\|_2^2 \leq \frac{N}{m}\lambda_{\max}(\frac{1}{N}\Tilde{\Ab}\Tilde{\Ab}^*)\|\cb^* - \cb^\sharp\|_2^2 \leq \frac{N(1+2\eta)}{m}\|\cb^* - \cb^\sharp\|_2^2
\end{aligned}
\end{equation*}
where the last inequality holds with probability at least $1-3\delta$ using Theorem \ref{Concentration_RF}. 
We now estimate $\|\cb^* - \cb^\sharp\|_2$. Recall that $\cb^\sharp = \Ab^\dagger\yb$ where the Moore-Penrose inverse $\Ab^\dagger = \Ab^*(\Ab\Ab^*)^{-1}$ of $\Ab$, then we apply triangle inequality 
\begin{equation*}
\|\cb^* - \cb^\sharp\|_2 \leq \|\Ab^\dagger\Ab\cb^* - \Ab^\dagger\yb\|_2 + 2\\|\Ab^\dagger\Ab\cb^* - \cb^*\|_2 \leq \|\Ab^\dagger\|_2\|\Ab\cb^*-\yb\|_2 + \|\Ab^\dagger\Ab - \Ib\|_2\|\cb^*\|_2.
\end{equation*}
The operator norm of the Moore-Penrose inverse is bounded by
\begin{equation*}
\|\Ab^\dagger\|_2 \leq \frac{1}{\sqrt{N\lambda_{\min}\left( \frac{1}{N}\Ab\Ab^*\right)}} \leq \frac{1}{\sqrt{N(1-2\eta)}}
\end{equation*}
with probability at least $1-3\delta$. 
Since $\Ab^\dagger\Ab-\Ib$ is an orthogonal projection, then its operator norm $\|\Ab^\dagger\Ab-\Ib\|_2 \leq 1$. 
It holds with probability at least $1-2\delta$ that
\begin{equation}
\label{c_c_star}
\begin{aligned}
& \|\cb^* - \cb^\sharp\|_2 \leq \frac{1}{\sqrt{N(1-2\eta)}} \|\Ab\cb^*-\yb\|_2 + \|\cb^*\|_2 \\
& \leq \frac{1}{\sqrt{N(1-2\eta)}} \left( \sum_{j=1}^{m} |f(\xb_j) - f^*(\xb_j)|^2\right)^{1/2} + \frac{4\sqrt{\log(2/\delta)}\|f\|_\rho}{\sqrt{N}} \\
& \leq \frac{\sqrt{m}}{\sqrt{N(1-2\eta)}} \frac{14\log(2/\delta)\|f\|_\rho}{\sqrt{N}} + \frac{4\sqrt{\log(2/\delta)}\|f\|_\rho}{\sqrt{N}},
\end{aligned}
\end{equation}
where we apply Lemma \ref{decay_c} in the second line and apply equation \ref{term1_x} in the last line.
Then, the second term is bounded by  
\begin{equation}
\label{term2}
\begin{aligned}
& \|f^*-f^\sharp\|_{L^2(\mu)} \leq \left( \frac{N(1+2\eta)}{m} + \frac{\sqrt{2\log\left(\frac{1}{\delta}\right)}N}{\sqrt{m}} \right)^{1/2} \|\cb^*-\cb^\sharp\|_2 \\
\leq &   \left( \frac{N(1+2\eta)}{m} + \frac{\sqrt{2\log\left(\frac{1}{\delta}\right)}N}{\sqrt{m}} \right)^{1/2} \left(  \frac{\sqrt{m}}{\sqrt{N(1-2\eta)}} \frac{14\log(2/\delta)\|f\|_\rho}{\sqrt{N}} + \frac{4\sqrt{\log(2/\delta)}\|f\|_\rho}{\sqrt{N}} \right) \\
\leq & \left( \frac{28\left(2m\log(1/\delta)\right)^{1/4}\log(2/\delta)}{\sqrt{N(1-2\eta)}} + \left(\frac{32\log(1/\delta)}{m}\right)^{1/4}\sqrt{\log(2/\delta)} \right) \|f\|_\rho
\end{aligned}
\end{equation}
with probability at least $1-8\delta$.
Adding \eqref{term1} and \eqref{term2} together gives the desired result.
\end{proof}

\bibliography{refs}

\end{document}

%% file: temp_marco.tex
\usepackage[utf8]{inputenc}
\usepackage{geometry}
\usepackage{setspace}

\usepackage{amsfonts,amsmath,amssymb,amsthm}
\usepackage[shortlabels]{enumitem}
\usepackage[hidelinks]{hyperref}
\usepackage{graphics,graphicx,subfigure}
\usepackage{caption}
\usepackage{romannum}
\usepackage{color}
\usepackage{tikz}
\usepackage{multirow}
\usepackage{bbm}

\usepackage{algorithm, xcolor}
\usepackage{algpseudocode}

\theoremstyle{plain}
\theoremstyle{definition}
\newtheorem{lemma}{Lemma}
\newtheorem{definition}[lemma]{Definition}
\newtheorem{theorem}[lemma]{Theorem}

\newtheorem{assumption}{Assumption}

\theoremstyle{remark}

\newcommand{\R}{\mathbb{R}}
\newcommand{\C}{\mathbb{C}}
\newcommand{\E}{\mathbb{E}}
\newcommand{\N}{\mathbb{N}}
\newcommand{\Pbb}{\mathbb{P}}

\newcommand{\Ab}{\mathbf{A}}

\newcommand{\Ib}{\mathbf{I}}

\newcommand{\Hb}{\mathbf{H}}

\newcommand{\xb}{\mathbf{x}}
\newcommand{\yb}{\mathbf{y}}
\newcommand{\cb}{\mathbf{c}}
\newcommand{\zb}{\mathbf{z}}
\newcommand{\ub}{\mathbf{u}}

\newcommand{\Uc}{\mathcal{U}}

\newcommand{\Nc}{\mathcal{N}}
\newcommand{\Fc}{\mathcal{F}}
\newcommand{\Xc}{\mathcal{X}}
\newcommand{\Mc}{\mathcal{M}}

\newcommand{\Tr}{\mathrm{\mathop{Tr}}}

\newcommand{\argmin}{\mathop{\mathrm{argmin}}}
\newcommand{\indicator}{\mathbbm{1}}
\newcommand{\omegab}{\boldsymbol\omega}